\pdfoutput=1

\documentclass[11pt]{article} 
\def\arXiv{1}

\usepackage{macros}
\notarxiv{
\usepackage[colorlinks=true,linkcolor=darkblue,%
citecolor=darkblue,urlcolor=darkblue]{hyperref}
}

\usepackage[noend]{algorithmic}
\usepackage{algorithm}

\usepackage{comment}

\newcommand{\dl}[1]{\textcolor{purple}{[\textbf{DL}: #1]}}
\newcommand{\f}{F}
\newcommand{\ff}{f}
\newcommand{\pf}{\ff}
\renewcommand{\S}{\mathbb{S}}
\newcommand{\diffp}{\varepsilon}  %
\newcommand{\ed}{\ensuremath{(\diffp,\delta)}}

\newcommand{\lkappa}{\underline{\kappa}} %

\usepackage{hyperref}
\usepackage{cleveref}

\newcommand{\A}{\mathsf{A}}
\newcommand{\cO}{\mathcal{O}}

\renewcommand{\ss}{\eta} %

\newcommand{\Ds}{\mathcal{S}} %
\newcommand{\ds}{s} %
\newcommand{\domain}{\mathbb{S}} %

\newcommand{\xdomain}{\mc{X}} %
\newcommand{\diam}{\mathsf{diam}} %
\newcommand{\opt}{^\star} %
\newcommand{\noise}{\zeta} %

\newcommand{\lip}{L}  %
\renewcommand{\sc}{\lambda} %
\newcommand{\rad}{D} %

\title{Adapting to Function Difficulty and \\
  Growth Conditions
  in Private Optimization}

\arxiv{ \author{Hilal Asi\thanks{ Equal contribution. Author order determined by
      coin toss.} ~~~~~~~ Daniel Levy$^*$ ~~~~~~~
    John C.\ Duchi\\
    \texttt{\{\href{mailto:asi@stanford.edu}{asi},%
      \href{mailto:danilevy@stanford.edu}{danilevy},%
		\href{mailto:jduchi@stanford.edu}{jduchi}%
		\}@stanford.edu
	}}
\date{}
}

\notarxiv{
\author{}
\date{}
}

\begin{document}
\maketitle

\begin{abstract}

  We develop algorithms for private stochastic convex optimization that adapt to
  the hardness of the specific function we wish to optimize. While previous work
  provide worst-case bounds for arbitrary convex functions, it is often the case
  that the function at hand belongs to a smaller class that enjoys faster
  rates. Concretely, we show that for functions exhibiting $\kappa$-growth
  around the optimum, i.e.,
  $\ff(x) \ge \ff(x^\star) + \lambda \kappa^{-1} \norms{x-x^\star}_2^\kappa$ for
  $\kappa > 1$, our algorithms improve upon the standard ${\sqrt{d}}/{n\diffp}$
  privacy rate to the faster
  $({\sqrt{d}}/{n\diffp})^{\tfrac{\kappa}{\kappa - 1}}$. Crucially, they achieve
  these rates without knowledge of the growth constant $\kappa$ of the
  function. Our algorithms build upon the inverse sensitivity mechanism, which
  adapts to instance difficulty~\cite{AsiDu20}, and recent localization
  techniques in private optimization~\cite{FeldmanKoTa20}. We complement our
  algorithms with matching lower bounds for these function classes and
  demonstrate that our adaptive algorithm is \emph{simultaneously} (minimax)
  optimal over all $\kappa \ge 1+c$ whenever $c = \Theta(1)$.

\end{abstract}

\section{Introduction}
\label{sec:intro}

Stochastic convex optimization (SCO) is a central problem in machine
learning and statistics, where for a sample space $\S$, parameter space
$\mc{X}\subset\R^d$, and a collection of convex losses $\crl{\f(\cdot;s):
  s\in\S}$, one wishes to solve
\begin{equation}\label{eq:sco}
  \minimize_{x\in\mc{X}} \ff(x) \defeq \E_{S\sim P}\brk*{\f(x;S)}
  = \int_{\S}\f(x;s)\mathrm{d}P(s)
\end{equation}
using an observed dataset $\mc{S} = S_1^n \simiid P$.  While as formulated, the
problem is by now fairly well-understood~\cite{BottouCuNo18, NemirovskiJuLaSh09,
  HazanKa11, BeckTe03, NemirovskiYu83}, it is becoming clear that, because of
considerations beyond pure statistical accuracy---memory or communication
costs~\cite{ZhangDuJoWa13_nips,GargMaNg14, BravermanGaMaNgWo16},
fairness~\cite{DworkHaPiReZe12,HashimotoSrNaLi18}, personalization or
distributed learning~\cite{McMahanMoRaHaAr17}---problem~\eqref{eq:sco} is simply
insufficient to address modern learning problems.  To that end, researchers have
revisited SCO under the additional constraint that the solution preserves the
privacy of the provided sample~\cite{DworkMcNiSm06, DworkKeMcMiNa06,
  AbadiChGoMcMiTaZh16, ChaudhuriMoSa11, DuchiJoWa13_focs}. A waypoint is
\citet{BassilySmTh14}, who provide a private method with optimal convergence
rates for the related empirical risk minimization problem, with recent papers
focus on SCO providing (worst-case) optimal rates in various settings: smooth
convex functions~\cite{BassilyFeTaTh19, FeldmanKoTa20}, non-smooth
functions~\cite{BassilyFeGuTa20}, non-Euclidean
geometry~\cite{AsiFeKoTa21,AsiDuFaJaTa21} and under more stringent privacy
constraints~\cite{LevySuAmKaKuMoSu21}.

Yet these works ground their analyses in worst-case scenarios and
provide guarantees for the \emph{hardest} instance of the class of problems they
consider. Conversely, they argue that their algorithms are optimal in a minimax
sense: for any algorithm, there exists a hard instance on which the
error achieved by the algorithm is equal to the upper bound. While valuable,
these results are pessimistic---the exhibited hard instances are typically
pathological---and fail to reflect achievable performance.

In this work, we consider the problem of adaptivity when
solving~\eqref{eq:sco} under privacy constraints. Importantly, we wish to
provide private algorithms that \emph{adapt} to the hardness of the
objective $\ff$. A loss function $\ff$ may belong to multiple problem
classes, each exhibiting different achievable rates, so a natural
desideratum is to attain the error rate of the easiest sub-class. As a
simple vignette, if one gets an arbitrary $1$-Lipschitz convex loss function
$\ff$, the worst-case guarantee of any $\diffp$-DP algorithm is
$\Theta(1/\sqrt{n} + d/(n\diffp))$. However, if one learns that $\ff$
exhibits some growth property---say $\ff$ is $1$-strongly convex---the
regret guarantee improves to the faster $\Theta(1/n + (d / (n\diffp))^2)$
rate with the appropriate algorithm. It is thus important to provide
algorithms that achieves the rates of the ``easiest'' class to which the
function belongs~\cite{JuditskyNe14, ZhuChDuLa16, DuchiRu21}.

To that end, consider the nested classes of functions $\mc{F}^\kappa$ for
$\kappa \in [1, \infty]$ such that, if $\ff \in \mc{F}^\kappa$ then there exists
$\lambda > 0$ such that for all $x \in \xdomain$,
\begin{equation*}
  \ff(x) -
  \inf_{x'\in\mc{X}} \ff(x') \ge \frac{\lambda}{\kappa}\norm{x-x^\star}_2^\kappa.
\end{equation*}
For example, strong convexity implies growth with
parameter $\kappa = 2$. This growth assumption closely relates to
uniform convexity~\cite{JuditskyNe14} and the
Polyak-Kurdyka-\L{}ojasiewicz inequality~\cite{BolteNgPeSu17}, and
we make these
connections precise in Section~\ref{sec:background}. Intuitively, smaller
$\kappa$ makes the function much easier to optimize: the error around the
optimal point grows quickly. Objectives with growth are widespread in machine
learning applications: among others, the $\ell_1$-regularized hinge loss
exhibits sharp growth (i.e.\ $\kappa = 1$) while $\ell_1$- or
$\ell_\infty$-constrained $\kappa$-norm regression~---i.e.\ $s = (a, b) \in
\R^d\times\R$ and $\f(x;s) = \abs{b - \tri{a, x}}^\kappa$---has
$\kappa$-growth for any $\kappa$ integer greater than
$2$~\cite{XuLiYa17}. In this work, we provide private adaptive algorithms
that adapt to the \emph{actual} growth of the function at hand.

We begin our analysis by examining Asi and Duchi's inverse sensitivity
mechanism~\cite{AsiDu20} on ERM as a motivation. While not a practical
algorithm, it achieves instance-optimal rates for any one-dimensional
function under mild assumptions, quantifying the best bound one could hope
to achieve with an adaptive algorithm, and showing (in principle) that
adaptive private algorithms can exist. We first show that for any function
with $\kappa$-growth, the inverse sensitivity mechanism achieves privacy
cost $(d/(n\diffp))^{\kappa / (\kappa - 1)}$; importantly, \emph{without
  knowledge of the function class $\mc{F}^\kappa$, that $\ff$ belongs
  to}. This constitutes grounding and motivation for our work in three ways:
(i) it validates our choice of sub-classes $\mc{F}^\kappa$ as the privacy
rate is effectively controlled by the value of $\kappa$, (ii) it exhibits
the rate we wish to achieve with efficient algorithms on $\mc{F}^\kappa$ and
(iii) it showcases that for easier functions,
privacy costs shrink significantly---to illustrate, for $\kappa =
5/4$ the privacy rate becomes $(d/(n\diffp))^5$.

We continue our treatment of problem~\eqref{eq:sco} under growth in
Section~\ref{sec:upper-bounds} and develop practical algorithms that achieve the
rates of the inverse sensitivity mechanism.  Moreover, for approximate
\ed-differential privacy, our algorithms improve the rates, achieving roughly
$(\sqrt{d}/(n\diffp))^{\kappa / (\kappa - 1)}$.  Our algorithms hinge on a
reduction to SCO: we show that by solving a sequence of increasingly constrained
SCO problems, one achieves the right rate whenever the function exhibits growth
at the optimum. Importantly, our algorithm only requires a \emph{lower bound}
$\underline{\kappa} \le \kappa$ (where $\kappa$ is the actual growth of $\ff$).

We provide optimality guarantees for our algorithms in Section~\ref{sec:lb} and
show that both the inverse sensitivity and the efficient algorithms of
Section~\ref{sec:upper-bounds} are \emph{simultaneously minimax optimal} over
all classes $\mc{F}^\kappa$ whenever $\kappa = 1+\Theta(1)$ and $d=1$ for
$\diffp$-DP algorithms. Finally, we prove that in \emph{arbitrary dimension},
for both pure- and approximate-DP constraints, our algorithms are also
simultaneously optimal for all classes $\mc{F}^\kappa$ with $\kappa \ge 2$.

On the way, we provide results that may be of independent
interest to the community. First, we develop optimal algorithms for SCO
under \emph{pure} differential privacy constraints, which, to the best of
our knowledge, do not exist in the literature. Secondly, our algorithms and
analysis provide high-probability bounds on the loss, whereas existing
results only provide (weaker) bounds on the expected loss.  Finally, we
complete the results of~\citet{RamdasSi13} on (non-private)
optimization lower bounds for functions with $\kappa$-growth by providing
information-theoretic lower bounds (in contrast to oracle-based
lower bounds that rely on observing only gradient information) and
capturing the optimal dependence on all problem parameters (namely $d, \lip$
and $\lambda$).

\subsection{Related work}
\label{sec:related-work}

Convex optimization is one of the best studied problems in private data
analysis~\cite{ChaudhuriMoSa11,DuchiJoWa13_focs,SmithTh13lasso,BassilySmTh14}.
The first papers in this line of work mainly study minimizing the empirical
loss, and readily establish that the (minimax) optimal privacy rates are
${d}/{n \diffp}$ for pure $\diffp$-DP and ${\sqrt{d \log(1/\delta)}}/{n \diffp}$
for $(\diffp,\delta)$-DP~\cite{ChaudhuriMoSa11,BassilySmTh14}. More recently,
several works instead consider the harder problem of privately minimizing the
population loss~\cite{BassilyFeTaTh19,FeldmanKoTa20}. These papers introduce new
algorithmic techniques to obtain the worst-case optimal rates of
$1/\sqrt{n}+ {\sqrt{d \log(1/\delta)}}/{n \diffp}$ for \ed-DP. They also show
how to improve this rate to the faster
$1/n + {{d \log(1/\delta)}}/{(n \diffp)^2}$ in the case of $1$-strongly convex
functions. Our work subsumes both of these results as they correspond to
$\kappa = \infty$ and $\kappa = 2$ respectively. To the best of our knowledge,
there has been no work in private optimization that investigates the rates under
general $\kappa$-growth assumptions or adaptivity to such conditions.

In contrast, the optimization community has extensively studied growth assumptions
~\cite{RamdasSi13,JuditskyNe14,ChatterjeeDuLaZh16} and show that on these
problems, carefully crafted algorithms improves upon the standard $1/\sqrt{n}$
for convex functions to the faster $(1/\sqrt{n})^{\kappa/(\kappa -
	1)}$. \cite{JuditskyNe14} derives worst-case optimal (in the first-order
oracle model) gradient algorithms in the uniformly convex case (i.e.
$\kappa \ge 2$) and provides technique to adapt to the growth $\kappa$, while
\cite{RamdasSi13}, drawing connections between growth conditions and active
learning, provides upper and lower bounds in the first-order stochastic oracle
model. We complete the results of the latter and provide
\emph{information-theoretic} lower bounds that have optimal dependence on
$d, \lambda$ and $n$---their lower bound only holding for $\lambda$ inversely
proportional to $d^{1/2-1/\kappa}$, when $\kappa \ge 2$. Closest to our work
is~\cite{ChatterjeeDuLaZh16} who studies instance-optimality via local minimax
complexity~\cite{CaiLo15}. For one-dimensional functions, they develop a
bisection-based instance-optimal algorithm and show that on individual functions
of the form $t \mapsto \kappa^{-1}\abs{t}^\kappa$, the local minimax rate is
$(1/\sqrt{n})^{\kappa / (\kappa - 1)}$.
%
%
%
%
%
%

\begin{comment}
Recently, there has been works that study adaptivity to the difficulty of the
underlying instance with differential privacy~\cite{AsiDu20,AsiDu20nips}. These
works develop an inverse sensitivity mechanism and show that it enjoys
instance-optimality guarantees for general functions.  However,
\citet{AsiDu20,AsiDu20nips} do not study the problem of interest in this paper,
that is, private optimization with growth conditions. Their techniques, though,
will prove useful in our setting as we show in~\Cref{sec:inv-sens}.\dl{this
reads a little weird; if it's instance-optimal, it should still be optimal for
our growth conditions?}
\end{comment}

%
%
%
%
%

%
%
%
%
 
%

\section{Preliminaries}
\label{sec:background}

We first provide notation that we use throughout this paper, define useful
assumptions and present key definitions in convex analysis and differential
privacy.

\paragraph{Notation.} $n$ typically denotes the sample size and $d$ the
dimension. Throughout this work, $x$ refers to the optimization variable,
$\mc{X} \subset \R^d$ to the constraint set and $s$ to elements ($S$ when random) of
the sample space $\S$. We usually denote by $\f:\mc{X}\times\S\to\R$ the
(convex) loss function and for a dataset
$\mc{S} = (s_1, \ldots, s_n) \subset \S$, we define the empirical and population
losses
\begin{equation*}
  \femp(x) \defeq \frac{1}{n}\sum_{i\le n}\f(x;s_i) \mbox{~~and~~} \ff(x) \defeq \E_{S\sim P}\brk*{\f(x;S)}.
\end{equation*}
We omit the dependence on $P$ as it is often clear from context. We reserve
$\diffp, \delta \ge 0$ for the privacy parameters of Definition~\ref{def:DP}. We
always take gradients with respect to the optimization variable $x$. In the case
that $\f(\cdot;s)$ is not differentiable at $x$, we override notation and define
$\nabla \f(x;s) = \argmin_{g\in \partial F(x;s)}\norm{g}_2$, where
$\partial \f(x;s)$ is the subdifferential of $\f(\cdot;s)$ at $x$. We use
$\msf{A}$ for (potentially random) mechanism and $S_1^n$ as a shorthand for
$\prn*{S_1, \ldots, S_n}$. For $p \ge 1$, $\norm{\cdot}_p$ is the standard
$\ell_p$-norm, $\ball_p^d(R)$ is the corresponding
$d$-dimensional $p$-ball of radius $R$ and $p^\star$ is the dual of $p$, i.e.\
such that $1/p^\star + 1/p = 1$. Finally, we define the Hamming distance between
datasets
$\dham(\mc{S}, \mc{S}') \defeq \inf_{\sigma \in \mathfrak{S}_n} \ones\crl{s_i
  \neq s'_{\sigma(i)}}$, where $\mathfrak{S}_n$ is the set of permutations over
sets of size $n$.

\paragraph{Assumptions.} We first state standard assumptions for
solving~\eqref{eq:sco}. We assume that $\mc{X}$ is a closed, convex domain such
that $\diam_2(\mc{X}) = \sup_{x, y \in \mc{X}}\norm{x-y}_2 \le \rad <
\infty$. Furthermore, we assume that for any $s\in\S$, $\f(\cdot;s)$ is convex
and $\lip$-Lipschitz with respect to $\norm{\cdot}_2$. Central to our work, we
define the following $\kappa$-growth assumption.

\begin{assumption}[$\kappa$-growth]\label{ass:growth}
  Let $x\opt = \argmin_{x\in\xdomain}\pf(x)$. For a loss $\f$ and distribution
  $P$, we say that $(\f, P)$ has $(\lambda, \kappa)$ growth for
  $\kappa\in[1, \infty]$ and $\lambda > 0$, if the population function satisfies
  \begin{equation*}
    \mbox{for all~} x\in\mc{X}, \quad f(x) - f(x^\star) \ge \frac{\lambda}{\kappa} \ltwo{x-x^\star}^\kappa.
  \end{equation*}
  In the case where $\what{P}$ is the empirical distribution on a finite dataset
  $\mc{S}$, we refer to $(\lambda, \kappa)$-growth of $(\f, \what{P})$ as
  $\kappa$-growth of the empirical function $\femp$.
\end{assumption}

\paragraph{Uniform convexity and \klinequality{} inequality.}
Assumption~\ref{ass:growth} is closely related to two fundamental notions in
convex analysis: uniform convexity and the \klinequality{}
inequality. Following~\cite{Nesterov08}, we say that
$h:\mc{Z}\subset\R^d \to \R$ is $(\sigma, \kappa)$-uniformly convex with
$\sigma > 0$ and $\kappa \ge 2$ if
\begin{equation*}
  \mbox{for all~~}x,y \in\mc{Z}, \quad h(y) \ge h(x) + \tri{\nabla h(x), y-x}
  + \frac{\sigma}{\kappa}\norm{x-y}_2^\kappa.
\end{equation*}
This immediately implies that (i) sums (and expectations) preserve uniform
convexity (ii) if $f$ is uniformly convex with $\lambda$ and $\kappa$, then it
has $(\lambda, \kappa)$-growth. This will be useful when constructing hard
instances as it will suffice to consider $(\lambda, \kappa)$-uniformly convex
functions which are generally more convenient to manipulate. Finally, we point
out that, in the general case that $\kappa \ge 1$, the literature refers to
Assumption~\ref{ass:growth} as the \klinequality{}
inequality~\cite{BolteNgPeSu17} with, in their notation,
$\varphi(s) = (\kappa / \lambda)^{1/\kappa}s^{1/\kappa}$. Theorem 5-(ii)
in~\cite{BolteNgPeSu17} says that, under mild conditions,
Assumption~\ref{ass:growth} implies the following inequality between the error
and the gradient norm for all $x \in \xdomain$
\begin{equation}\label{eq:kl}
  f(x) - \inf_{x'\in\mc{X}} f(x')
  \le
  \frac{e}{\lambdae}\nrm*{\nabla f(x)}_2^{\tfrac{\kappa}{\kappa - 1}},
\end{equation}
This is a key result in our analysis of the inverse sensitivity mechanism of
Section~\ref{sec:inv-sens}.

\paragraph{Differential privacy.} We begin by recalling the definition of
\ed-differential privacy.
\begin{definition}[\cite{DworkMcNiSm06,DworkKeMcMiNa06}]
	\label{def:DP}
	A randomized  algorithm $\A$ is \ed-differentially private (\ed-DP) if, for all datasets $\Ds,\Ds' \in \domain^n$ that differ in a single data element and for all events $\cO$ in the output space of $\A$, we have
	\[
	\Pr \left( \A(\Ds)\in \cO \right) \leq e^{\eps} \Pr \left(\A(\Ds')\in \cO \right) +\delta.
	\]
\end{definition}

We use the following standard results in differential privacy.
\begin{lemma}[Composition~{\cite[][Thm.~3.16]{DworkRo14}}]
	\label{lemma:basic-comp}    
	If $\A_1,\dots,\A_k$ are randomized algorithms that each is
        $(\diffp, \delta)$-DP, then their composition
        $(\A_1(\Ds),\dots,\A_k(\Ds))$ is $(k \diffp, k\delta)$-DP.
\end{lemma}

Next, we consider the Laplace mechanism. We will let $Z \sim \laplace_d(\sigma)$ denote a $d$-dimensional vector $Z \in \R^d$ such that $Z_i \simiid \laplace(\sigma)$ for $1 \le i \le d$.
\begin{lemma}[Laplace mechanism~{\cite[][Thm.~3.6]{DworkRo14}}]
	\label{lemma:lap-mech}
	Let $h: \domain^n \to \R^d$ have $\ell_1$-sensitivity $\Delta$,
	that is, $\sup_{\Ds,\Ds' \in \domain^n: \dham(\Ds,\Ds') \le 1} \lone{h(\Ds) - h(\Ds')} \le \Delta$. Then the Laplace mechanism $\A(\Ds) = h(\Ds) + \laplace_d(\sigma)$ with $\sigma = \Delta/\diffp$ is $\diffp$-DP.
\end{lemma}

Finally, we need the Gaussian mechanism for \ed-DP.
\begin{lemma}[Gaussian mechanism~{\cite[][Thm.~A.1]{DworkRo14}}]
	\label{lemma:gauss-mech}
	Let $h: \domain^n \to \R^d$ have $\ell_2$-sensitivity $\Delta$,
	that is, $\sup_{\Ds,\Ds' \in \domain^n: \dham(\Ds,\Ds') \le 1} \ltwo{h(\Ds) - h(\Ds')} \le \Delta$. Then the Gaussian mechanism $\A(\Ds) = h(\Ds) + \normal(0,\sigma^2 I_d)$ with $\sigma = 2 \Delta \log(2/\delta)/\diffp$ is $(\diffp,\delta)$-DP.
\end{lemma}

\paragraph{Inverse sensitivity mechanism.} Our goal is to design private
optimization algorithms that adapt to the difficulty of the underlying
function. As a reference point, we turn to the inverse sensitivity mechanism
of~\cite{AsiDu20} as it enjoys general instance-optimality guarantees.  For a
given function $h: \domain^n \to \mathcal{T} \subset \R^d$ that we wish to
estimate privately, define the \emph{inverse sensitivity} at $x\in\mc{T}$
\begin{equation}
\label{eq:inv-sens}
	\invmodcont_h(\Ds;x) = \inf_{\Ds'} \{ \dham(\Ds',\Ds) : h(\Ds')=x  \},
\end{equation}
that is, the inverse sensitivity of a target parameter $y \in \mathcal{T}$ at
instance $\Ds$ is the minimal number of samples one needs to change to reach a
new instance $\Ds'$ such that $h(\Ds') = y$. Having this quantity, the inverse
sensitivity mechanism samples an output from the following probability density
\begin{equation}
\label{eq:inv-sens}
\pdf_{\Ainvsm(\Ds)} (x) \propto e^{- \diffp \invmodcont_h(\Ds;x) }.
\end{equation}
The inverse sensitivity mechanism preserves $\diffp$-DP and enjoys instance-optimality guarantees in general settings~\cite{AsiDu20}. In contrast to (worst-case) minimax optimality guarantees which measure the performance of the algorithm on the hardest instance, these notions of instance-optimality provide stronger per-instance optimality guarantees.

\section{Adaptive rates through inverse sensitivity for $\diffp$-DP}
\label{sec:inv-sens}

To understand the achievable rates when privately optimizing functions with
growth, we begin our theoretical investigation by examining the inverse
sensitivity mechanism in our setting. 
We show that, for instances that exhibit $\kappa$-growth of the
empirical function, the inverse sensitivity mechanism privately solves ERM with
excess loss roughly $({d}/{n \diffp} )^{\frac{\kappa}{\kappa-1}}$.

In our setting, we use a gradient-based approximation of the inverse sensitivity mechanism to simplify the analysis, while attaining similar rates. Following~\cite{AsiDu20nips} with our function of interest $h(\Ds) \defeq \argmin_{x \in \xdomain} \pf_\Ds(x)$, we can lower bound the inverse sensitivity $\invmodcont_h(\Ds;x) \ge n \ltwo{\nabla f_\Ds(x)}/2 \lip$ under natural assumptions. We define a $\rho$-smoothed version of this quantity which is more suitable to continuous domains
\begin{equation*}
G^\rho_\Ds(x) = \inf_{y \in \xdomain : \ltwo{y-x} \le \rho} \ltwo{\nabla f_\Ds(y)},
\end{equation*}
and define the $\rho$-smooth gradient-based inverse sensitivity mechanism
\begin{equation}
\label{eq:grad-inv-sens}
\pdf_{\Agrinvsm(\Ds)} (x) \propto e^{- \diffp n G^\rho_\Ds(x)/2 \lip }.
\end{equation}

Note that while exactly sampling from the un-normalized density
$\pdf_{\Agrinvsm\prn{\Ds}}$ is computationally intractable, analyzing its
performance is an important step towards understanding the optimal rates for the
family of functions with growth that we study in this work. The following
theorem demonstrates the adaptivity of the inverse sensitivity mechanism to the
growth of the underlying instance. We defer the proof
to~\Cref{sec:apdx-inv-sens}.
\begin{restatable}{theorem}{restateInvSens}
	\label{thm:inv-sens}
	Let $\Ds=(\ds_1, \ldots, \ds_n)\in \domain^n$, $\f(x;\ds)$ be convex, $\lip$-Lipschitz for all $\ds \in \domain$.
	Let $x\opt = \argmin_{x \in \xdomain} \pf_\Ds(x)$ and assume $x\opt$ is in the interior of $\xdomain$. Assume that $\pf_S(x)$ has $\kappa$-growth (Assumption~\ref{ass:growth}) with $\kappa \ge \lkappa > 1$.   
	For $\rho>0$, the $\rho$-smooth inverse sensitivity mechanism $\Agrinvsm$~\eqref{eq:grad-inv-sens} is $\diffp$-DP, and with probability at least $1-\beta$ 
	the output $\hat x = \Agrinvsm(\Ds) $ has
	\begin{equation*}
	 \pf_\Ds(\hat x) - \min_{x \in \xdomain} \pf_\Ds(x) 
	\le  \frac{1}{\lambda^{\frac{1}{\kappa-1}}} \left(\frac{2 \lip (\log(1/\beta) + d \log(\rad/\rho))}{n \diffp} \right)^{\frac{\kappa}{\kappa-1}} + \lip \rho.
	\end{equation*}
	Moreover, setting $\rho = ({\lip}/{\sc})^{\frac{1}{\lkappa-1}} ({d}/{n \diffp})^{\frac{\lkappa}{\lkappa-1}}$, we have
	\begin{align*}
	\pf_\Ds(\hat x) - \min_{x \in \xdomain} \pf_\Ds(x) 
	& \le  \frac{1}{\lambda^{\frac{1}{\kappa-1}}} \wt O \left( \frac{\lip d}{n \diffp} \right)^{\frac{\kappa}{\kappa-1}}.
	\end{align*}
\end{restatable}

The rates of the inverse sensitivity in~\Cref{thm:inv-sens} provide two main
insights regarding the landscape of the problem with growth conditions. First,
these conditions allow to improve the worst-case rate $d/n\diffp$ to
$({d}/{n \diffp} )^{\frac{\kappa}{\kappa-1}}$ for pure $\diffp$-DP and therefore
suggest a better rate
$({\sqrt{d \log(1/\delta)}}/{n \diffp} )^{\frac{\kappa}{\kappa-1}}$ is possible
for approximate \ed-DP. Moreover, the general instance-optimality guarantees of
this mechanism~\cite{AsiDu20} hint that these are the optimal rates for our
class of functions. In the sections to come, we validate the correctness of
these predictions by developing efficient algorithms that achieve these rates
(for pure and approximate privacy), and prove matching lower bounds which
demonstrate the optimality of these algorithms.

\section{Efficient algorithms with optimal rates}
\label{sec:upper-bounds}

While the previous section demonstrates that there exists algorithms that
improve the rates for functions with growth, we pointed out that $\Agrinvsm$ was
computationally intractable in the general case. In this section, we develop
efficient algorithms---e.g.\ that are implementable with gradient-based
methods---that achieve the same convergence rates. Our algorithms build on the
recent localization techniques that~\citet{FeldmanKoTa20} used to obtain optimal
rates for DP-SCO with general convex functions. In Section~\ref{sec:dp-sco-hb},
we use these techniques to develop private algorithms that achieve the optimal
rates for (pure) DP-SCO with high probability, in contrast to existing results
which bound the expected excess loss. These results are of independent interest.

In Section~\ref{sec:dp-sco-growth}, we translate these results into convergence
guarantees on privately optimizing convex functions with growth by solving a
sequence of increasingly constrained SCO problems---the high-probability
guarantees of~\Cref{sec:dp-sco-hb} being crucial to our convergence analysis of
these algorithms.

\subsection{High-probability guarantees for convex DP-SCO}
\label{sec:dp-sco-hb}

We first describe our algorithm (Algorithm~\ref{alg:pure-erm}) then analyze its
performance under pure-DP (Proposition~\ref{thm:sco-hb-pure}) and approximate-DP
constraints (Proposition~\ref{thm:sco-hb-appr}). Our analysis builds on novel
tight generalization bounds for uniformly-stable algorithms with high
probability~\cite{FeldmanVo19}. We defer the proofs
to~\Cref{sec:apdx-sec:upper-bounds}.

\begin{algorithm}
	\caption{Localization-based Algorithm}
	\label{alg:pure-erm}
	\begin{algorithmic}[1]
		\REQUIRE Dataset $\Ds=(\ds_1, \ldots, \ds_n)\in \domain^n$,
		constraint set $\xdomain$,
		step size $\ss$, initial point $x_0$,
		privacy parameters $(\diffp,\delta)$;
		\STATE Set $k = \ceil{\log n}$ and $n_0 = n/k$
		\FOR{$i=1$ to $k$\,}
		\STATE Set  
		$\ss_i = 2^{-4i} \ss$ 
		\STATE Solve the following ERM over 
		$ \mc{X}_i= \{x\in \xdomain: \norm{x - x_{i-1}}_2 \le {2\lip \ss_i n_0} \}$:\\
		$%
		\displaystyle
		\quad \quad \quad \quad 
		F_i(x) =  \frac{1}{n_0} \sum_{j=1 + (i-1)n_0}^{in_0} \f(x;\ds_j) + \frac{1}{\ss_i n_0 } \norm{x - x_{i-1}}_2^2
		$%
		\STATE Let $\hat x_i$ be the output of the optimization algorithm
		\IF{$\delta=0$}
		\STATE Set 
		$\noise_i \sim \laplace_d(\sigma_i)$ where $\sigma_i = 4 \lip \ss_i \sqrt{d}/\diffp_i$
		\ELSIF{$\delta>0$}
		\STATE Set $\noise_i \sim \normal(0,\sigma_i^2)$ where $\sigma_i = 4 \lip \ss_i \sqrt{\log(1/\delta)}/\diffp$
		\ENDIF
		\STATE Set $x_i = \hat x_i + \noise_i$
		\ENDFOR
		\RETURN the final iterate $x_k$
	\end{algorithmic}
\end{algorithm}

\begin{restatable}{proposition}{restatePureSCOHb}
 	\label{thm:sco-hb-pure}
 	Let $\beta \le 1/(n+d)$, $\diam_2(\xdomain)\le \rad$ and $\f(x;\ds)$ be convex, $\lip$-Lipschitz for all $\ds \in \domain$. Setting
 	\begin{equation*}
 	\ss = \frac{\rad}{\lip} \min \left(\frac{1}{\sqrt{n \log(1/\beta)}} ,\frac{ \diffp}{ d \log(1/\beta)} \right)
 	\end{equation*}
 	then for $\delta=0$, \Cref{alg:pure-erm} is $\diffp$-DP and has with probability $1-\beta$
 	\begin{equation*}
 	f(x) -  f(x\opt) 
 	\le \lip \rad \cdot  O \left( \frac{\sqrt{\log(1/\beta) } \log^{3/2} n}{\sqrt{n}} + \frac{ d \log(1/\beta) \log n}{n \diffp} \right).
 	\end{equation*}
\end{restatable} 
 
Similarly, by using a different choice for the parameters and noise distribution, we have the following guarantees for approximate \ed-DP.
 \begin{restatable}{proposition}{restateApproxSCOHb}
 	\label{thm:sco-hb-appr}
 	Let $\beta \le 1/(n+d)$, $\diam_2(\xdomain)\le \rad$ and $\f(x;\ds)$ be convex, $\lip$-Lipschitz for all $\ds \in \domain$. Setting
 	\begin{equation*}
 	\ss = \frac{\rad}{\lip} \min \left(\frac{1}{\sqrt{n \log(1/\beta)}} ,\frac{ \diffp}{ \sqrt{d \log(1/\delta)} \log(1/\beta)} \right),
 	\end{equation*}
 	then for $\delta > 0 $, \Cref{alg:pure-erm} is $(\diffp,\delta)$-DP and has with probability $1-\beta$
 	\begin{equation*}
 	f(x) -  f(x\opt) 
 	\le \lip \rad \cdot  O \left( \frac{\sqrt{\log(1/\beta) } \log^{3/2} n}{\sqrt{n}} + \frac{ \sqrt{d \log(1/\delta)} \log(1/\beta) \log n}{n \diffp} \right).
 	\end{equation*}
 \end{restatable}

\subsection{Algorithms for DP-SCO with growth}
\label{sec:dp-sco-growth}
Building on the algorithms of the previous section, we design algorithms that
recover the rates of the inverse sensitivity mechanism for functions with
growth, importantly without knowledge of the value of $\kappa$. Inspired by
epoch-based algorithms from the optimization
literature~\cite{JuditskyNe10,HazanKa11}, our algorithm iteratively applies the
private procedures from the previous section. Crucially, the growth assumption
allows to reduce the diameter of the domain after each run, hence improving the
overall excess loss by carefully choosing the hyper-parameters. We provide full
details in~\Cref{alg:loc-growth}.

\begin{algorithm}
	\caption{Epoch-based algorithms for $\kappa$-growth}
	\label{alg:loc-growth}
	\begin{algorithmic}[1]
		\REQUIRE 
		Dataset $\Ds=(\ds_1, \ldots, \ds_n)\in \domain^n$,
		convex set $\xdomain$,
		initial point $x_0$,
		number of iterations $T$,
		privacy parameters $(\diffp,\delta)$;
		\STATE Set $n_0 = n/T$ and $\rad_0 = \diam_2(\xdomain)$
		\IF{$\delta=0$}
		\STATE Set 
		$\ss_0 = \frac{\rad_0}{2\lip} \min \left(\frac{1}{\sqrt{n_0 \log(n_0) \log(1/\beta)}} ,\frac{ \diffp}{ d \log(1/\beta)} \right)$ 
		\ELSIF{$\delta>0$}
		\STATE $\ss_0 = \frac{\rad_0}{2\lip} \min \left(\frac{1}{\sqrt{n_0 \log(n_0) \log(1/\beta)}} ,\frac{ \diffp}{ \sqrt{d \log(1/\delta)} \log(1/\beta)} \right)$
		\ENDIF
		\FOR{$i=0$ to $T - 1$\,}
		\STATE Let $\Ds_i = (\ds_{1+(i-1) n_0}, \dots, \ds_{i n_0})$
		\STATE Set $\rad_i = 2^{-i} \rad_0$ and $\ss_i = 2^{-i} \ss_0$
		\STATE Set $\xdomain_i = \crl{x \in \xdomain : \ltwo{x - x_i} \le \rad_i}$
		\STATE Run~\Cref{alg:pure-erm} on dataset $\Ds_i$ with starting point $x_i$, privacy parameter $(\diffp,\delta)$, domain $\xdomain_i$ (with diameter $\rad_i$), step size $\ss_i$ 
		\STATE Let $x_{i+1}$ be the output of the private procedure 
		\ENDFOR
		\RETURN $x_T$
	\end{algorithmic}
	\label{Alg:pure-with-growth}
\end{algorithm} 

The following theorem summarizes our main upper bound for DP-SCO with growth in the pure privacy model, recovering the rates of the inverse sensitivity mechanism in~\Cref{sec:inv-sens}. We defer the proof to~\Cref{sec:proof-thm-sco-growth-pure}.
\begin{restatable}{theorem}{restateSCOGrowthPure}
	\label{thm:sco-growth-pure}
	Let $\beta \le 1/(n+d)$, $\diam_2(\xdomain)\le \rad$ and $\f(x;\ds)$ be convex, $\lip$-Lipschitz for all $\ds \in \domain$. 
	Assume that $\pf$ has $\kappa$-growth (Assumption~\ref{ass:growth}) with $\kappa \ge \lkappa >1$.
	Setting $T = \left\lceil \frac{2 \log n}{\lkappa-1} \right\rceil$,	
	\Cref{alg:loc-growth} is $\diffp$-DP and has with probability $1-\beta$
	\begin{equation*}
	\pf(x_T) -  \min_{x \in \xdomain} \pf(x) 
	\le  \frac{1}{\lambda^{\frac{1}{\kappa - 1}}} \cdot  \wt O \left( \frac{\lip \sqrt{\log(1/\beta) } }{\sqrt{n}} + \frac{ \lip d \log(1/\beta)}{n \diffp (\lkappa -1)} \right)^{\frac{\kappa}{\kappa-1}},
	\end{equation*}
	where $\wt O$ hides logarithmic factors depending on $n$ and $d$.
\end{restatable} 

\begin{proof}[Sketch of the proof]
  The main challenge of the proof is showing that the iterate achieves good risk
  without knowledge of $\kappa$. Let us denote by $D\cdot\rho$ the error
  guarantee of Proposition~\ref{thm:sco-hb-pure} (or
  Proposition~\ref{thm:sco-hb-appr} for approximate-DP). At each stage $i$, as
  long as $x^\star = \argmin_{x\in\mc{X}} \ff(x)$ belongs to $\mc{X}_i$, the
  excess loss is of order $D_i\cdot\rho$ and thus decreases exponentially fast
  with $i$. The challenge is that, without knowledge of $\kappa$, we do not know
  the index $i_0$ (roughly $\tfrac{\log_2 n}{\kappa - 1}$) after which
  $x^\star \notin D_j$ for $j\ge i_0$ and the regret guarantees become
  meaningless with respect to the original problem. However, in the stages after
  $i_0$, as the constraint set becomes very small, we upper bound the variations
  in function values $f(x_{j+1}) - f(x_j)$ and show that the sub-optimality
  cannot increase (overall) by more than $O(D_{i_0}\cdot\rho)$, thus achieving
  the optimal rate of stage $i_0$.

\end{proof}

Moreover, we can improve the dependence on the dimension for approximate \ed-DP, resulting in the following bounds.

\begin{restatable}{theorem}{restateSCOGrowthApprox}
	\label{thm:sco-growth-appr}
	Let $\beta \le 1/(n+d)$, $\diam_2(\xdomain)\le \rad$ and $\f(x;\ds)$ be convex, $\lip$-Lipschitz for all $\ds \in \domain$. 
	Assume that $\pf$ has $\kappa$-growth (Assumption~\ref{ass:growth}) with $\kappa \ge \lkappa >1$.
	Setting $T = \left\lceil \frac{2 \log n}{\lkappa-1} \right\rceil$ and $\delta>0$, \Cref{alg:loc-growth} is $(\diffp,\delta)$-DP and has with probability $1-\beta$
	\begin{equation*}
	\pf(x_T) -  \min_{x \in \xdomain} \pf(x)
	\le \frac{1}{\lambda^{\frac{1}{\kappa - 1}}} \cdot \wt O \left( \frac{\lip \sqrt{\log(1/\beta) } }{\sqrt{n}} + \frac{ \lip \sqrt{d \log(1/\delta)} \log(1/\beta)}{n \diffp (\lkappa -1)} \right)^{\frac{\kappa}{\kappa-1}},
	\end{equation*}
	where $\wt O$ hides logarithmic factors depending on $n$ and $d$.
\end{restatable} 

\section{Lower bounds}\label{sec:lb}

In this section, we develop (minimax) lower bounds for the problem of SCO with
$\kappa$-growth under privacy constraints. Note that taking $\diffp \to \infty$
provides lower bound for the unconstrained minimax risk. For a sample space
$\S$ and collection of distributions $\mc{P}$ over $\S$, we define the function
class $\mc{F}^\kappa(\mc{P})$ as the set of convex functions from $\R^d\to\R$
that are $\lip$-Lipschitz and has $\kappa$-growth
(Assumption~\ref{ass:growth}). 
We define the \emph{constrained} minimax risk~\cite{BarberDu14a}
\begin{equation}
  \mathfrak{M}_n(\mc{X}, \mc{P}, \mc{F}^\kappa, \diffp, \delta) \defeq \inf_{\what{x}_n\in\mc{A}^{\diffp, \delta}}
  \sup_{(\f, P) \in \mc{F}^\kappa\times\mc{P}} \E\brk*{
    \ff(\what{x}_n(S_1^n)) - \inf_{x'\in\mc{X}}f(x')
  },
\end{equation}
where $\mc{A}^{\epsilon, \delta}$ is the collection of $(\diffp, \delta)$-DP
mechanisms from $\S^n$ to $\mc{X}$. When clear from context, we omit the
dependency on $\mc{P}$ of the function class and simply write
$\mc{F}^\kappa$. We also forego the dependence on $\delta$ when referring to
pure-DP constraints, i.e.
$\mathfrak{M}_n(\mc{X}, \mc{P}, \mc{F}^\kappa, \diffp, \delta=0) \eqdef
\mathfrak{M}_n(\mc{X}, \mc{P}, \mc{F}^\kappa, \diffp)$. We now proceed to prove
tight lower bounds for $\diffp$-DP in~\Cref{sec:lb-pure} and \ed-DP
in~\Cref{sec:lb-appr}.

\subsection{Lower bounds for pure $\diffp$-DP}\label{sec:lb-pure}

Although in~\Cref{sec:upper-bounds} we show that the same algorithm achieves the
optimal upper bounds for all values of $\kappa >1$, the landscape of the problem
is more subtle for the lower bounds and we need to delineate two different cases
to obtain tight lower bounds. We begin with $\kappa \ge 2$,
which %
corresponds to uniform convexity and enjoys properties that make the problem
easier (e.g., closure under summation or addition of linear terms).  The second
case, $1 < \kappa <2$, corresponds to sharper growth and requires a different
hard instance to satisfy the growth
condition.%

\paragraph{$\kappa$-growth with $\kappa \ge 2$. } We begin by developing
lower bounds under pure DP for $\kappa \ge 2$ %

\begin{restatable}[Lower bound for $\diffp$-DP, $\kappa\ge 2$]
  {theorem}{restateLbPrivateLargeKappa}
  \label{thm:lb-sco-private-large-kappa}
    Let $d \ge 1$, $\mc{X} = \ball_2^d(R)$, $\S = \crl{\pm e_j}_{j\le d}$,
  $\kappa \ge 2$ and $n\in\N$. Let $\mc{P}$ be the set of distributions on
  $\S$. Assume that
  \begin{equation*}
    2^{\kappa-1} 
    \le
    \frac{\lip}{\sc} \frac{1}{ R^{\kappa - 1}}
    \le   2^{\kappa-1} 
    \sqrt{96 n} \mbox{~~and~~} n\diffp \ge \frac{1}{\sqrt{3}}
  \end{equation*}
    
  The following lower bound holds
  \begin{equation}\label{eq:lb-private-large-kappa}
	\mathfrak{M}_n(\mc{X}, \mc{P}, \mc{F}^\kappa, \epsilon) \ge
	\frac{1}{\lambda^{\tfrac{1}{\kappa - 1}}}\tilde{\Omega}\prn*{
	\prn*{\frac{L}{\sqrt{n}}}^{\tfrac{\kappa}{(\kappa-1)}}
	+
	\prn*{\frac{Ld}{n\diffp}}^{\tfrac{\kappa}{\kappa - 1}}
}.
	\end{equation}
\end{restatable}

First of all, note that $L \ge \lambda 2^\kappa R^{\kappa - 1}$ is not an
overly-restrictive assumption. Indeed, for an arbitrary
$(\lambda, \kappa)$-uniformly convex and $\lip$-Lipschitz function, it always
holds that $L \ge \tfrac{\lambda}{2}R^{\kappa - 1}$. This is thus equivalent to
assuming $\kappa = \Theta(1)$. Note that when $\kappa \gg 1$, the standard
$n^{-1/2} + d/(n\diffp)$ lower bound holds. We present the proof
in~\Cref{app:lb-pure-large-kappa} and preview the main ideas here.
\begin{proof}[Sketch of the proof]
  Our lower bounds hinges on the collections of functions
  $F(x;s) \defeq a\kappa^{-1}\norm{x}_2^\kappa + b\tri{x, s}$ for $a, b\ge 0$ to
  be chosen later. These functions are~\cite[][Lemma 4]{Nesterov08}
  $\kappa$-uniformly convex for any $s\in\S$ and in turn, so is the population
  function $\ff$. We proceed as follows, we first prove an information-theoretic
  (non-private) lower bound (\Cref{thm:lb-non-private-large-kappa} in
  Appendix~\ref{app:lb-pure-large-kappa}) which provides the statistical term
  in~\eqref{eq:lb-private-large-kappa}. With the same family of functions, we
  exhibit a collection of datasets and prove by contradiction that if an
  estimator were to optimize below a certain error it would have violated
  $\diffp$-DP---this yields a lower bound on ERM for our function class
  (Theorem~\ref{thm:lb-private-large-kappa} in
  Appendix~\ref{app:lb-pure-large-kappa}). We conclude by proving a reduction
  from SCO to ERM in~\Cref{sec:lb-erm-sco}.
\end{proof}

\paragraph{$\kappa$-growth with $\kappa \in (1, 2]$.} 
As the construction of the hard instance is more intricate for
$\kappa < 2$, we provide a one-dimensional lower bound and leave the
high-dimensional case for future work. 
In this case we directly obtain the result with a
private version of Le Cam's method~\cite{Yu97,Wainwright19,BarberDu14a}, however
with a different family of functions.

The issue with the construction of the previous section is that the function
does not exhibit sharp growth for $\kappa < 2$. Indeed, the added linear
function shifts the minimum away from $0$ where the function is differentiable
and as a result it locally behaves as a quadratic and only achieves growth
$\kappa = 2$. To establish the lower bound, we consider a different sample
function $\f$ that has growth exactly $1$ on one side and $\kappa$ on the other
side. This yields the following

\begin{restatable}[Lower bound for $\diffp$-DP, $\kappa \in {(1, 2]}$ 
	]{theorem}{restateLbSmallKappa}\label{thm:lb-small-kappa}
   Let $d=1$, $\S = \crl{-1, +1}$, $\kappa \in (1, 2]$, $\lambda=1$, $\lip=2$, and $n\in\N$. There exists a collection
  of distributions $\mc{P}$ such that, whenever $n\diffp \ge 1/\sqrt{3}$, it holds that
  \begin{equation}
	\mathfrak{M}_n(\brk{-1, 1}, \mc{P}, \mc{F}^\kappa_{d=1}, \epsilon) = \Omega
	\crl*{\prn*{\frac{1}{\sqrt{n}}}^{\tfrac{\kappa}{\kappa - 1}} +  \prn*{\frac{1}{n \diffp}}^{\tfrac{\kappa}{\kappa - 1}}}.
	\end{equation}
\end{restatable}

\subsection{Lower bounds under approximate privacy constraints}
\label{sec:lb-appr}

We conclude our treatment by providing lower bounds but now under
\emph{approximate} privacy constraints, demonstrating the optimality of the risk
bound of Theorem~\ref{thm:sco-growth-appr}. We prove the result via a
reduction: we show that if one solves ERM with $\kappa$-growth with
error $\Delta$, this implies that one solves arbitrary convex ERM with error
$\phi(\Delta)$. Given that a lower bound of $\Omega\prn{\sqrt{d}/(n\diffp)}$
holds for ERM, a lower bound of $\phi^{-1}\prn{\sqrt{d}/(n\diffp)}$ holds for
ERM with $\kappa$-growth. However, for this reduction to hold, we require that
$\kappa \ge 2$. Furthermore, we consider $\kappa$ to be roughly a constant---in
the case that $\kappa$ is too large, standard lower bounds on general convex
functions apply.

\begin{theorem}[Private lower bound for $(\diffp, \delta)$-DP]
	\label{thm:lb-appr}
  Let $\kappa \ge 2$ such that $\kappa = \Theta(1)$, $\mc{X} =
  \ball_2^d(\rad)$. Let $d \ge 1$ and $\S = \crl{\pm 1/\sqrt{d}}^d$. Assume that
  $n\diffp = \Omega(\sqrt{d})$, then for any $(\diffp, \delta)$ mechanism
  $\msf{A}$, there exists $\lambda > 0, \f$ and $\mc{S} \subset \S$ such that
  \begin{equation*}
    \E\brk{\femp(\msf{A}(\mc{S}))} - \inf_{x'\in\mc{X}}\femp(x') \ge \tilde{\Omega}\brk*{\frac{1}{\lambdae}
      \prn*{\frac{L\sqrt{d}}{n\diffp}}^{\tfrac{\kappa}{\kappa-1}}
    }.
  \end{equation*}
\end{theorem}

\Cref{thm:lb-appr} implies that the same lower bound (up to logarithmic factors) applies to SCO
via the reduction of~\cite[][Appendix C]{BassilyFeTaTh19}. Before proving the
theorem, let us state (and prove in~\Cref{app:prf-lb-approx}) the following
reduction: if an \ed-DP algorithm achieves excess error (roughly) $\Delta$ on
ERM for any function with $\kappa$-growth, there exists an \ed-DP algorithm that
achieves error $\Delta^{(\kappa-1) / \kappa}$ for any convex function. We
construct the latter by iteratively solving ERM problems with geometrically
increasing $\norm{\cdot}_2^\kappa$-regularization towards the previous iterate
to ensure the objective has $\kappa$-growth.

\begin{restatable}[Solving ERM with $\kappa$-growth implies solving any convex
  ERM]{proposition}{restatePropReduction}\label{prop:reduction-growth-to-erm}
  Let $\kappa \ge 2$. Assume there
  exists an $(\epsilon, \delta)$ mechanism $\msf{A}$ such that for any
  $\lip$-Lipschitz loss $G$ on $\mc{Y}$ and dataset $\mc{S}$ such that
  $g_{\mc{S}}(x) \defeq \frac{1}{n}\sum_{s\in\mc{S}} G(x;s)$ exhibits
  $(\lambda, \kappa)$-growth, the mechanism achieves excess loss
  \begin{equation*}
    \E\brk{\gemp(\msf{A}(\mc{S}, G, \mc{Y}))} - \inf_{y'\in\mc{Y}}\gemp(y')
    \le \frac{1}{\lambdae}\Delta(n, L, \epsilon, \delta).
  \end{equation*}

  Then, we can construct an $(\diffp, \delta)$-DP mechanism $\msf{A'}$ such that
  for any $\lip$-Lipschitz loss $\ff$, the mechanism achieves excess loss
  \begin{equation*}
    \E\brk{\femp(\msf{A}'(\mc{S}))} - \inf_{x'\in\mc{X}}\femp(x')
    \le O\prn*{D\brk*{\Delta(n, L, \epsilon / k, \delta/k)}^{\tfrac{\kappa - 1}{\kappa}}},
  \end{equation*}
  where $k$ is the smallest integer such that
  $k \ge \log\brk*{\frac{\kappa^{\tfrac{1}{\kappa -
          1}}L^{\tfrac{\kappa}{\kappa-1}}}{2^{2\kappa - 3}\Delta(n, L, \diffp/k,
      \delta/k)}}$.
\end{restatable}

With this proposition, the proof of the theorem directly follows
as~\citet{BassilySmTh14} prove a lower bound $\Omega(\sqrt{d}/(n\diffp))$ for
ERM with \ed-DP.
\section*{Discussion}

In this work, we develop private algorithms that adapt to the growth of the
function at hand, achieving the convergence rate corresponding to the
``easiest'' sub-class the function belongs to. However, the picture is not yet
complete. First of, there are still gaps in our theoretical understanding, the
most interesting one being $\kappa = 1$. On these functions, appropriate
optimization algorithms achieve linear convergence~\cite{XuLiYa17} and raise the
question, can we achieve exponentially small privacy cost in our setting?
Finally, while our optimality guarantees are more fine-grained than the usual
minimax results over convex functions, they are still contigent on some
predetermined choice of sub-classes. Studying more general notions of adaptivity
is an important future direction in private optimization.

\arxiv{

  \subsection*{Acknowledgments}
  The authors would like to thank Karan Chadha and Gary Cheng for comments on an
  early version of the draft.  }

\bibliographystyle{abbrvnat}

\notarxiv{
\clearpage

\section*{Checklist}

\begin{enumerate}

\item For all authors...
\begin{enumerate}
  \item Do the main claims made in the abstract and introduction accurately reflect the paper's contributions and scope?
    \answerYes{}
  \item Did you describe the limitations of your work?
    \answerYes{}
  \item Did you discuss any potential negative societal impacts of your work?
    \answerYes{}
  \item Have you read the ethics review guidelines and ensured that your paper conforms to them?
    \answerYes{}
\end{enumerate}

\item If you are including theoretical results...
\begin{enumerate}
  \item Did you state the full set of assumptions of all theoretical results?
    \answerYes{}
	\item Did you include complete proofs of all theoretical results?
    \answerYes{}
\end{enumerate}

\item If you ran experiments...
\begin{enumerate}
  \item Did you include the code, data, and instructions needed to reproduce the main experimental results (either in the supplemental material or as a URL)?
    \answerNA{}
  \item Did you specify all the training details (e.g., data splits, hyperparameters, how they were chosen)?
    \answerNA{}
	\item Did you report error bars (e.g., with respect to the random seed after running experiments multiple times)?
    \answerNA{}
	\item Did you include the total amount of compute and the type of resources used (e.g., type of GPUs, internal cluster, or cloud provider)?
    \answerNA{}
\end{enumerate}

\item If you are using existing assets (e.g., code, data, models) or curating/releasing new assets...
\begin{enumerate}
  \item If your work uses existing assets, did you cite the creators?
    \answerNA{}
  \item Did you mention the license of the assets?
    \answerNA{}
  \item Did you include any new assets either in the supplemental material or as a URL?
    \answerNA{}
  \item Did you discuss whether and how consent was obtained from people whose data you're using/curating?
    \answerNA{}
  \item Did you discuss whether the data you are using/curating contains personally identifiable information or offensive content?
    \answerNA{}
\end{enumerate}

\item If you used crowdsourcing or conducted research with human subjects...
\begin{enumerate}
  \item Did you include the full text of instructions given to participants and screenshots, if applicable?
    \answerNA{}
  \item Did you describe any potential participant risks, with links to Institutional Review Board (IRB) approvals, if applicable?
    \answerNA{}
  \item Did you include the estimated hourly wage paid to participants and the total amount spent on participant compensation?
    \answerNA{}
\end{enumerate}

\end{enumerate}

\clearpage

\section*{Potential negative societal impact}

The aim of our work is theoretical in essence and as such, we do not expect
direct negative societal impact. As DP becomes more establish as a norm, we
believe this research is relevant for practitionners in both industry and
government. Indeed, an important obstacle to applying DP is the loss of
performance compared to non-private models; our theoretical results suggests
that better adaptive algorithms would significantly narrow this performance
gap. We wish to point out two potential negative consequences of growing
research in privacy. First, a simple but effective method to guarantee privacy
is to either delete existing user data or limit data collection in the first
place. Paradoxically, the more confident institutions are in DP algorithms, the
less they are susceptible to turn to these simpler---and most
effective---solutions. Finally, using DP algorithms should not preclude one from
(1) carefully choosing $\eps$ and $\delta$ to provide meaningful guarantees for
the specific application at hand and (2) developing exhaustive and meticulous
evaluation methods of the privacy of deployed models.

}

\clearpage

\appendix

\section{Proofs for~\Cref{sec:inv-sens}}
\label{sec:apdx-inv-sens}

\subsection{Proof of~\Cref{thm:inv-sens}}

\restateInvSens*

Let us first prove privacy. The sensitivity of $\ltwo{\nabla  f_\Ds(x)}$ is $2\lip/n$ as $F$ is $\lip$-Lipschitz, therefore following the privacy proof of the smooth inverse sensitivity mechanism~\cite[Prop. 3.2]{AsiDu20} we get that $\Agrinvsm$~\eqref{eq:grad-inv-sens} is $\diffp$-DP.

Let us now prove the claim about utility. Denote $\hat x = \Agrinvsm(\Ds) $ and $E = \frac{ 2 \lip K}{n \diffp} $ with $K$ to be chosen presently. We argue that it is enough to show that $\Pr(G_\rho(\hat x)  \ge E) \le \beta $. Indeed then with probability at least $1-\beta$ we have $G_\rho(\hat x) \le  E$, which implies there is $y$ such that $\ltwo{\hat x -y}\le \rho$ and $\ltwo{\nabla f_\Ds(y)} \le E$,
hence using the \klinequality{} inequality~\eqref{eq:kl}
\begin{align*}
	 f_\Ds(\hat x) - f_\Ds(x\opt) 
	& = f_\Ds(\hat x) -  f_\Ds(y) +  f_\Ds(y) -  f_\Ds(x\opt) \\
	& \le \lip \rho +   \frac{e}{\lambda^{\frac{1}{\kappa-1}}} \ltwo{\nabla  f_\Ds(y)}^{\frac{\kappa}{\kappa-1}} \\
	& \le \lip \rho +  \frac{e}{\lambda^{\frac{1}{\kappa-1}}} E^{\frac{\kappa}{\kappa-1}}.
\end{align*}
	
It remains to prove that $\Pr(G_\rho(\hat x) \ge  E) \le \beta $.
Let $S_0 = \{ x \in \R^d : \ltwo{x - x\opt} \le \rho  \}$ and $S_1 = \{ x \in \R^d :  G_\rho(x) \ge  E \}$. Note that $G_\rho(x) = 0$ for any $x \in S_0$ as $x\opt$ is in the interior of $\xdomain$ which implies $\nabla f_\Ds(x\opt) = 0$. Hence the definition of the smooth inverse sensitivity mechanism~\eqref{eq:grad-inv-sens} implies
\begin{align*}
	\Pr( \Agrinvsm(\Ds) \in S_1) 
	& \le \frac{\vol(\{x \in \R^d : \ltwo{x - x\opt} \le \rad + \rho  \}) e^{-\frac{n \diffp}{2 \lip} E}}{ \vol(\{x \in \R^d : \ltwo{x - x\opt} \le \rho  \})} \\
	& \le e^{-K}  \left(1 + \frac{D}{\rho} \right)^d \le \beta,
\end{align*}
where the last inequality follows by choosing $K = \log(1/\beta) + d \log(1 + \rad/\rho)$.

\section{Proofs for~\Cref{sec:upper-bounds}}
\label{sec:apdx-sec:upper-bounds}

We need to the following result on the generalization properties of uniformly stable algorithms~\cite{FeldmanVo19}.
\begin{theorem}~\cite[Cor. 4.2]{FeldmanVo19}
	\label{thm:stab-hp}
	Assume $\diam_2(\xdomain) \le \rad$.	
	Let $\Ds = (S_1,\dots,S_n)$ where $S_1^n \simiid P$ and $F(x;\ds)$ is $\lip$-Lipschitz and $\sc$-strongly convex for all $\ds \in \domain$.
	Let $\hat x = \argmin_{x \in \xdomain} {f_\Ds(x)}$ be the empirical minimizer.
	For $0 < \beta \le 1/n$, with probability at least $1-\beta$
	\begin{equation*}
	f(\hat x) - f(x\opt) 
	\le \frac{c \lip^2 \log(n) \log(1/\beta)}{\sc n} + \frac{c \lip \rad \sqrt{\log(1/\beta)}}{\sqrt{n}}.
	\end{equation*} 
\end{theorem} 

\subsection{Proof of~\Cref{thm:sco-hb-pure}}

\restatePureSCOHb*

We begin by proving the privacy claim. We show that each iterate is $\diffp$-DP and therefore post-processing implies the claim as each sample is used in exactly one iterate. To this end, let $\lambda_i = 1/\ss_i n_0$ and note that the minimizer $\hat x_i$ has $\ell_2$ sensitivity $2\lip/\lambda_i n_0 \le 4 \lip \ss_i$~\cite{FeldmanKoTa20}, hence the $\ell_1$-sensitivity is at most $4 \lip \ss_i \sqrt{d}$. Standard properties of the Laplace mechanism~\cite{DworkRo14} now imply that $x_i$ is $\diffp$-DP which give the claim about privacy. 
	
Now we proceed to prove utility which follows similar arguments to the localization-based proof in~\cite{FeldmanKoTa20}.
Letting $\hat x_0 = x\opt$, we have:
\begin{align*}
	 \pf(x_k) - \pf(x\opt)
	& = \sum_{i=1}^k  \pf(\hat x_i) - \pf(\hat x_{i-1})
	+ \pf(x_k) - \pf(\hat x_k).
\end{align*}
First, by using standard properties of Laplace distributions~\cite{Duchi19}, we know that for $\noise_i \sim \laplace(\sigma_i)$, 
\begin{equation*}
\Pr(\ltwo{\noise_i} \ge t )
	\le \Pr(\linf{\noise_i} \ge t/\sqrt{d} )
	\le d e^{-t/\sqrt{d} \sigma_i},
\end{equation*}
which implies (as $\beta \le 1/(n+d)$) that with probability $1-\beta/2$ we have $\ltwo{\noise_i} \le 10\sqrt{d} \sigma_i \log(1/\beta)$ for all $1 \le i \le k$.
Hence
\begin{align*}
	\pf  (x_k) - \pf  (\hat x_k)
	& \le \lip  \ltwo{x_k - \hat x_k} \\
	& \le \lip \sigma_k \sqrt{d} \log(1/\beta) \\
	& \le 4 \lip^2 d \frac{\ss_i }{\diffp} \\
	& \le 4 \lip^2 d \frac{\ss }{\diffp 2^{4i}} 
	 \le \frac{4 \lip \rad}{n^2},
\end{align*}
where the last inequality follows since
$\ss = \frac{\rad \diffp}{\lip d \log(k/\beta)}$. 
Now we use high-probability generalization guarantees of uniformly-stable algorithms.  
We use~\Cref{thm:stab-hp} with $F(x;\ds_j) + \frac{\ltwo{x- x_{i-1}}^2}{\ss_i n_0} $ to get that with probability $1-\beta/2$ for each $i$
\begin{align*}
f(\hat x_i) -  f(\hat x_{i-1})
\le \frac{\ltwo{\hat x_{i-1} - x_{i-1}}^2}{\ss_i n_0} + {c \lip^2 \log(n) \log(1/\beta) \ss_i} + \frac{c \lip \rad \sqrt{\log(1/\beta)}}{\sqrt{n_0}}.
\end{align*}
Thus, 
\begin{align*}
  \sum_{i=1}^k \pf  (\hat x_i) - \pf (\hat x_{i-1})
  & \le \sum_{i=1}^k \crl*{\frac{\norm{\hat x_{i-1} - x_{i-1}}_2^2 }{\ss_i n_0} + {c \lip^2 \log(n) \log(1/\beta) \ss_i} + \frac{c \lip \rad \sqrt{\log(1/\beta)}}{\sqrt{n_0}}} \\
  & \le \frac{\rad^2}{\ss n_0} + \brk*{\sum_{i=2}^k \frac{\sigma_{i-1}^2 d \log^2(1/\beta)}{\ss_i n_0}} + {2 c \lip^2 \log(n) \log(1/\beta) \ss} + \frac{c \lip \rad \sqrt{\log(1/\beta)} k}{\sqrt{n_0}}
  \\
  & = \frac{\rad^2 }{\ss n_0}  + \brk*{\sum_{i=2}^k  \frac{C \lip^2 \ss_{i-1} d^2 \log^2(1/\beta)}{ n_0 \diffp^2}}  + {2 c \lip^2 \log(n) \log(1/\beta) \ss} + \frac{c \lip \rad \sqrt{\log(1/\beta)} k}{\sqrt{n_0}} \\
  & = \frac{\rad^2}{\ss n_0} + \frac{C \lip^2 \ss d^2 \log^2(1/\beta)}{ n_0 \diffp^2}\brk*{\sum_{i=2}^k 2^{-i}}  + {2 c \lip^2 \log(n) \log(1/\beta) \ss} + \frac{c \lip \rad \sqrt{\log(1/\beta)} k}{\sqrt{n_0}}
  \\
  & \le \lip \rad \cdot  O \left( \frac{ \sqrt{\log(1/\beta)  \log(n)  } + \sqrt{\log(1/\beta)} \log^{3/2}(n) }{\sqrt{n}} + \frac{d \log(1/\beta) \log(n)}{n \diffp} \right),
\end{align*}
where the last inequality follows by choosing $
\ss = \frac{\rad}{\lip} \min \left(\frac{1}{\sqrt{n \log(1/\beta)}} ,\frac{ \diffp}{ d \log(1/\beta)} \right)$ 

\subsection{Proof of~\Cref{thm:sco-hb-appr}}

\restateApproxSCOHb*

The proof is similar to the proof of~\Cref{thm:sco-hb-pure}.
For privacy, we show in the proof of~\Cref{thm:sco-hb-pure} that 
the $\ell_2$-sensitivity of $\hat x_i$ is upper bounded by $2\lip/\lambda_i n_0 \le 4 \lip \ss_i$ hence standard properties of the Gaussian mechanism~\cite{DworkRo14} now imply that $x_i$ is $(\diffp,\delta)$-DP which implies the final algorithm is \ed-DP using post-processing.

The utility proof follows the same arguments as in the proof
of~\Cref{thm:sco-hb-pure}, except that for $\noise_i \sim \normal(0,\sigma_i^2)$
we have~\cite{JinNeGeKaJo19} (since $\noise_i$ is
$2\sqrt{2}\sigma_i\sqrt{d}$-norm-sub-Gaussian)%
\begin{equation*}
\Pr(\ltwo{\noise_i} \ge t \sqrt{d}  )
\le 2 e^{-\tfrac{t^2}{16\sigma_i^2}},
\end{equation*}
implying that $\ltwo{\noise_i} \le 4\sqrt{d} \sigma_i \log(4/\beta)$ for all
$1 \le i \le k$ with probability $1-\beta/2$.

\subsection{Proofs of~Theorems~\ref{thm:sco-growth-pure} and~\ref{thm:sco-growth-appr}}
\label{sec:proof-thm-sco-growth-pure}

We first restate Theorems~\ref{thm:sco-growth-pure} and~\ref{thm:sco-growth-appr}.

\restateSCOGrowthPure*

\restateSCOGrowthApprox*

We start by proving privacy. Since each sample $\ds_i$ is used in exactly one iterate, we only need to show that each iterate is $(\diffp,\delta)$-DP, which will imply the main claim using post-processing. The privacy of each iterate follows directly from the privacy guarantees of~\Cref{alg:pure-erm}. We proceed to prove utility.

We will prove the utility claim assuming the subroutine	used in~\Cref{alg:loc-growth} satisfies the following: the output $x_{k+1}$ has error 
\begin{equation*}
	\pf(x_{k+1}) - \min_{x \in \xdomain} \pf(x) \le \rad_k \cdot \rho, 
\end{equation*}	
for some $\rho >0$. Note that in our setting, \Cref{thm:sco-hb-pure} implies that $\rho \le \lip \cdot  O ( \frac{\sqrt{\log(1/\beta) } \log n_0}{\sqrt{n_0}} + \frac{ d \log(1/\beta)}{n_0 \diffp} )$ for pure-DP and similarly~\Cref{thm:sco-hb-appr} gives the corresponding $\rho$ for \ed-DP.

The proof has two stages. In the first stage (\Cref{lemma:stage-i}), we prove that as long as $i \le i_0$ for some $i_0 >0$, then $x\opt \in \xdomain_i$ and the performance of the algorithm keeps improving. We show that at the end of this stage, the points $x_{i_0+1}$ has optimal excess loss. Then, in the second stage (\Cref{lemma:stage-ii}), we show that the iterates would not move much as the radius $\rad_i$ of the domain is sufficiently small, hence the final accumulated error along these iterations is small.

Let us begin with the first stage. Let $i_0$ be the largest $i$ such that $\rad_i \ge (\frac{\kappa 2^\kappa \rho}{\sc })^{\frac{1}{\kappa-1}}$.
We prove that $x\opt \in \xdomain_i$ for all $0 \le i \le i_0$ where we 
recall that  $\xdomain_i = \crl{x \in \xdomain : \ltwo{x - x_i} \le \rad_i}$ and 
$\rad_i = 2^{-i} \rad_0$.

\begin{lemma}
\label{lemma:stage-i}
	For all $0 \le i \le i_0$ we have
	\begin{equation*}
	x\opt \in \xdomain_i
	\quad \mbox{~~and~~} \quad
	\pf(x_{i_0+1}) - \min_{x \in \xdomain} \pf(x) 
		\le 4({2^\kappa})^{\frac{1}{\kappa-1}} \frac{1}{\sc^{\frac{1}{\kappa-1}}} \rho^{\frac{\kappa}{\kappa-1}}.
	\end{equation*}
\end{lemma}	
\begin{proof}
	To prove the first part, we need to show that $\ltwo{x_i - x\opt} \le \rad_i$. 
	Let $\bar \rad_i = \ltwo{x_i - x\opt}$.
	First, note that the claim is true for $i=0$. Now we assume it is correct for $0 \le i \le i_0 - 1$ and prove correctness for $i+1$.
	Note that the growth condition implies
	\begin{equation*}
	\bar \rad_{i+1} \le (\kappa \Delta_i/\sc)^{1/\kappa},
	\end{equation*}
	where $\Delta_i = \pf(x_{i+1}) - \min_{x\in \xdomain} \pf(x) \le \rad_i \cdot \rho$. %
	Thus we have
	\begin{equation*}
	\bar \rad_{i+1} 
		\le (\kappa \rad_i \rho/\sc)^{1/\kappa} 
		\le \rad_i /2 = \rad_{i+1},
	\end{equation*}
	where the second inequality holds for $i$ that satisfies $\rad_i \ge (\frac{\kappa 2^\kappa \rho}{\sc })^{\frac{1}{\kappa-1}}$. This proves the first part of the claim.
	For the second part, note that the definition of $i_0$ implies that $D_{i_0} \le 2(\frac{\kappa 2^\kappa \rho}{\sc })^{\frac{1}{\kappa-1}}$. Therefore, as $x\opt \in \xdomain_{i_0}$ and the algorithm has error $\rad_i \cdot \rho$, we have
	\begin{align*}
	\pf(x_{i_0+1}) - \min_{x \in \xdomain} \pf(x) 
		& \le \rad_{i_0} \cdot \rho \\
		& \le 2({\kappa 2^\kappa}/{\sc })^{\frac{1}{\kappa-1}}    \rho^{\frac{\kappa}{\kappa-1}}.
	\end{align*}
	The claim now follows as $\kappa^{\frac{1}{\kappa-1}} \le 2$.
\end{proof}

We now proceed to the second stage. The following lemma shows that the accumulated error along the iterates $i > i_0$ is small and therefore $x_T$ obtains the same error as $x_{i_0+1}$ (up to constant factors).
\begin{lemma}
\label{lemma:stage-ii}
	Assume the algorithm has error $\rad_i \cdot \rho$.
	Let $i_0$ be the largest $i$ such that $\rad_i \ge (\frac{\kappa 2^\kappa \rho}{\sc })^{\frac{1}{\kappa-1}}$. For all $i \ge i_0 + 1$ we have
	\begin{equation*}
	\pf(x_{i+1}) - \pf(x_i) 
	\le 2^{-(i-i_0)} \rad_{i_0}  \rho.
	\end{equation*}
	In particular, for $T \ge i_0 + 1$ we have
	\begin{equation*}
	\pf(x_{T}) - \min_{x \in \xdomain} \pf(x) 
	\le 8 ({ 2^\kappa}/{\sc })^{\frac{1}{\kappa-1}}    \rho^{\frac{\kappa}{\kappa-1}}.
	\end{equation*}
\end{lemma}	

\begin{proof}
	Note that as $x_i \in \xdomain_i$, the guarantees of the algorithm give
	\begin{equation*}
	\pf(x_{i+1}) - \pf(x_i) 
		\le \rad_{i}  \rho 
		= 2^{-(i-i_0)} \rad_{i_0} \rho.
	\end{equation*}
	For the second part of the claim, we have
	\begin{align*}
	\pf(x_{T}) - \min_{x \in \xdomain} \pf(x)  
		& = \pf(x_{i_0+1}) - \min_{x \in \xdomain} \pf(x)  + \sum_{i = i_0 + 1}^{T} \pf(x_{i+1}) - \pf(x_{i}) \\
		& \le  \rad_{i_0} \rho  + \sum_{i = i_0 + 1}^{T}  2^{-(i-i_0)} \rad_{i_0} \rho
		\le 2 \rad_{i_0} \rho.
	\end{align*}
	The claim now follows as $D_{i_0} \le 2(\frac{\kappa 2^\kappa \rho}{\sc })^{\frac{1}{\kappa-1}}$ and $\kappa^{\frac{1}{\kappa-1}} \le 2$.
\end{proof}

Assuming $T \ge i_0 + 1$, \Cref{thm:sco-growth-pure} and~\Cref{thm:sco-growth-appr} now follow immediately from~\Cref{lemma:stage-ii}. Indeed, for the case of pure-DP ($\delta=0$), the choice of hyper-parameters in~\Cref{alg:loc-growth} and the guarantees of~\Cref{alg:pure-erm} (\Cref{thm:sco-hb-pure}) imply that $\rho \le \lip \cdot  O ( \frac{\sqrt{\log(1/\beta) } \log n_0}{\sqrt{n_0}} + \frac{ T d \log(1/\beta)}{n_0 \diffp} )$, which proves \Cref{thm:sco-growth-pure}. Similarly, \Cref{thm:sco-growth-appr} follows by using the guarantees of of~\Cref{alg:pure-erm} for approximate \ed-DP, that ism \Cref{thm:sco-hb-appr}, which gives $\rho \le \lip  \cdot  O \left( \frac{\sqrt{\log(1/\beta) } \log n_0}{\sqrt{n_0}} + \frac{ T \sqrt{d \log(1/\delta)} \log(1/\beta)}{n_0 \diffp} \right)$. Note that our choice of stepsize at each iterate implies that~\Cref{thm:sco-growth-pure} guarantees the desired utility with probability at least $1-\beta^2$, hence the final utility guarantee holds with probability at least $1 - T \beta^2 \ge 1 - \beta$.

It remains to verify $T \ge i_0 + 1$. Note that by choosing $T \ge \frac{2 \log(\rad_0^{\kappa-1} \sc/\rho)}{\lkappa-1}$, we get that $\rad_T \le (\frac{\kappa 2^\kappa \rho}{\sc })^{\frac{1}{\kappa-1}}$, hence $T \ge i_0 + 1$. As we have $\rho \ge \lip/\sqrt{n_0}$ (non-private error) and $\rad_0^{\kappa - 1} \le \lip/\sc$ in our setting, we get that choosing $T = \frac{2 \log n}{\lkappa-1}$ gives the claim.

\section{Proofs of Section~\ref{sec:lb}}

In this section, we provide the proofs for our lower bound under privacy
constraints for functions with growth. This section is organized as follows: we
prove in~\Cref{app:prf-lb-pure}, the lower bounds under pure-DP and
in~\Cref{app:prf-lb-approx}, the lower bounds under
approximate-DP. Within~\Cref{app:prf-lb-pure}, we distinguish between
$\kappa \ge 2$ (\Cref{app:lb-pure-large-kappa}) and $\kappa \in \prn{1, 2}$
(\Cref{app:prf-thm-lb-small-kappa}).

\subsection{Proofs of Section~\ref{sec:lb-pure}}\label{app:prf-lb-pure}

\subsubsection{Proof of~\Cref{thm:lb-sco-private-large-kappa}}
\label{app:lb-pure-large-kappa}

As we preview in the main text, the proof combines the (non-private)
information-theoretic lower bounds of~\Cref{thm:lb-non-private-large-kappa} with
the (private) lower bound on ERM of~\Cref{thm:lb-private-large-kappa}. Finally,
we show in~\Cref{sec:lb-erm-sco} that privately solving SCO is harder than
privately solving ERM, concluding the proof of the theorem. We restate the
theorem and prove these results in sequence.

\restateLbPrivateLargeKappa*{}

\paragraph{Non-private lower bound} We begin the proof
of~\Cref{thm:lb-sco-private-large-kappa} by proving a (non-private)
information-theoretic lower bound for minimizing functions with $\kappa\ge
2$-growth. We use the standard reduction from estimation to
testing~\cite[see][Appendix A.1]{LevyDu19} in conjunction with Fano's
method~\cite{Wainwright19,Yu97}.

\begin{theorem}[Non-private lower bound]\label{thm:lb-non-private-large-kappa}
    Let
    $d \ge 1$, $\mc{X} = \ball_2^d(R)$, $\S = \crl{\pm e_j}_{j\le d}$, $\kappa \ge 2$
    and $n\in\N$. Let $\mc{P}$ be the set of distributions on $\S$. Assume that
    \begin{equation*}
      2^{\kappa-1} 
      \le
      \frac{\lip}{\sc} \frac{1}{ R^{\kappa - 1}}
      \le   2^{\kappa-1} 
      \sqrt{96 n}.
    \end{equation*}
    
    The following lower bound holds
    \begin{equation*}
    \mathfrak{M}_n(\mc{X}, \mc{P}, \mc{F}^\kappa) \gtrsim \frac{1}{\lambda^{\tfrac{1}{\kappa - 1}}}
    \prn*{\frac{L}{\sqrt{n}}}^{\tfrac{\kappa}{(\kappa-1)}}.
  \end{equation*}
\end{theorem}

\begin{proof}
  For $\mc{V} \subset \crl{ \pm 1}^d$ let us consider the following function and
  distribution
  \begin{equation*}
    \f(x;s) \defeq \frac{\lambda 2^{\kappa - 2}}{\kappa}\norm{x}_2^\kappa + \frac{L}{2}\tri{x, s} \mbox{~~and~~}
    X \sim P_v \mbox{~~implies~~} X_j = \begin{cases}
      v_je_j \mbox{~~w.p.~~} \frac{1+\delta}{2} \\
      -v_j e_j \mbox{~~w.p.~~} \frac{1-\delta}{2}.
    \end{cases}    
  \end{equation*}
  Since the linear term does not affect uniform convexity, Lemma~4
  in~\cite{Nesterov08} guarantees that $\ff_v$ is $(\lambda, \kappa)$-uniformly
  convex. Furthermore, for $s \in \S$
  \begin{equation*}
    \norm{\nabla \f(x;s)}_2 \le \lambda 2^{\kappa - 2}R^{\kappa - 1} + \frac{L}{2} \le L,
  \end{equation*}
  by assumption, so the functions are $\lip$-Lipschitz
  and satisfy Assumption~\ref{ass:growth}.

  \underline{Computing the separation.} As $\E_{P_v}S = \tfrac{\delta}{d}v$, we have
  \begin{equation*}
    \ff_v(x) = \frac{\lambda 2^{\kappa - 2}}{\kappa}\norm{x}_2^\kappa + \frac{L\delta}{2d}\tri{x, v}.
  \end{equation*}
  Note that for $u\in\R^d, \sigma > 0$, it holds that
  \begin{equation*}
    \inf_{x\in\R^d}\sigma\frac{\norm{x}_2^\kappa}{\kappa} + \tri{x, u} = -\frac{1}{\kappa^\star}
    \prn*{\frac{1}{\sigma}}^{\tfrac{1}{\kappa - 1}}\norm{u}^{\tfrac{\kappa}{\kappa - 1}}
    \mbox{~~at~~} x_u^\star = -\prn*{\frac{1}{\sigma}}^{\tfrac{1}{\kappa - 1}}
    \prn*{\frac{1}{\norm{u}_2}}^{\tfrac{\kappa - 2}{\kappa - 1}}u.
  \end{equation*}
  To make sure that $x_u^\star \in \ball_2^d(R)$, we require
  $\norms{u}_2 \le \sigma R^{\kappa - 1}$. After choosing $\delta$, we will see
  that this holds under the assumptions of the theorem. Let us consider the
  Gilbert-Varshimov packing of the hypercube: there exists
  $\mc{V}\subset \crl*{\pm 1}^d$ such that $\abs{\mc{V}} = \exp(d/8)$ and
  $\dham(v, v') \ge d/4$ for all $v\neq v' \in\mc{V}$. Let us compute the separation
  \begin{equation*}
    \inf_{x\in\ball_2^d(R)} \frac{f_v(x) + f_{v'}(x)}{2} = -\frac{1}{4\kappa^\star\lambda^{\tfrac{1}{\kappa-1}}}
    \prn*{\frac{L\delta}{d}}^{\tfrac{\kappa}{\kappa-1}}\nrm*{\frac{v+v'}{2}}_2^{\tfrac{\kappa}{\kappa-1}}
  \end{equation*}
  Note that $\norm{(v+v')/2}_2 = \sqrt{d - \dham(v, v')} \le \sqrt{3d/4}$. This yields a separation
  \begin{equation*}
    d_{\mathsf{opt}}(v, v', \mc{X}) \ge \frac{1-(3/4)^{\kappa/(2\kappa - 2)}}{2\kappa^\star\lambda^{\tfrac{1}{\kappa - 1}}}
    \prn*{\frac{L\delta}{\sqrt{d}}}^{\tfrac{\kappa}{\kappa-1}}.
\end{equation*}

\underline{Lower bounding the testing error.} In the case of a multiple
hypothesis test, we use Fano's method and for $V \sim \uniform\crl{\mc{V}}$ and
$S_1^n | V = v \simiid P_v$, Fano's inequality guarantees
\begin{equation*}
  \inf_{\psi:\S^n\to\mc{V}}\P(\psi(S_1^n)\neq V) \ge 1 - \frac{\msf{I}(S_1^n;V) + \log 2}{\log\abs{\mc{V}}},
\end{equation*}
where $\msf{I}(X;Y)$ is the Shannon mutual information between $X$ and $Y$. In
our case, we have $\log\abs{\mc{V}} \ge d/8$ and
$\msf{I}(S_1^n;V) \le n\max_{v\neq v'}\dkls{P_v}{P_{v'}} \le 3n\delta^2$. In the
case $d \ge 48\log 2$, we choose $\delta = \sqrt{d / (24n)}$. We handle the
one-dimensional case thereafter. For this $\delta$, we have
\begin{equation*}
  \mathfrak{M}_n(\mc{X}, \mc{P}, \mc{F}^\kappa) \ge \frac{1 - \prn*{\tfrac{3}{4}}^{\tfrac{\kappa}{2\kappa - 2}}}
  {4\kappa^\star (24)^{\tfrac{\kappa}{2\kappa - 2}}}
  \frac{1}{\lambda^{\tfrac{1}{\kappa - 1}}}
  \prn*{\frac{L^2}{n}}^{\tfrac{\kappa}{2\kappa - 2}}.
\end{equation*}

For this choice of $\delta$, the assumption on $n$ ensures that the minimum
remains in $\ball_2^d(R)$.

\underline{One-dimensional lower bound with Le Cam's method.} Since Fano's method
requires $d \ge 48\log 2$, we finish the proof by providing a lower bound for
$d=1$ using Le Cam's method. We use the same family of functions in one
dimension, i.e. $\S = \crl{\pm 1}$, $v\in \crl{\pm 1}$ and for
$\delta \in \brk{0, 1}$ define
\begin{equation*}
  \f(x;s) = \frac{\lambda2^{\kappa - 2}}{\kappa}\abs{x}^\kappa + \frac{L}{2}s\cdot x \mbox{~~and~~} X \sim P_v \mbox{~~implies~~} X = \begin{cases}
    v & \mbox{~~w.p.~~} \frac{1+\delta}{2} \\
    -v & \mbox{~~w.p.~~} \frac{1-\delta}{2}.
  \end{cases}
\end{equation*}

As this is the one-dimensional analog of the previous construction, $\f$ remains
$L$-lipschitz and $\ff$ has $(\lambda, \kappa)$-growth. A calculation yields
that the separation is
\begin{equation*}
  d_{\msf{opt}}(1, -1, \mc{X}) \ge \frac{1}{2\lambdae}\prn*{L\delta}^{\tfrac{\kappa}{\kappa - 1}},
\end{equation*}
where we used that $\kappa^\star \in \brk{1, 2}$. For $V\sim\uniform\crl{-1, 1}$
and $S_1^n | V = v \simiid P_v$. Le Cam's lemma in conjunction with Pinsker's
inequality yields that
    \begin{equation*}
      \inf_{\psi:\S^n\to\crl{-1, 1}} \P(\psi(S_1^n) \neq V)
      = \frac{1}{2}(1 - \tvnorm{P_1^n - P_{-1}^n})
      \ge \frac{1}{2}(1 - \sqrt{\tfrac{n}{2}\dkls{P_1}{P_{-1}}}).
    \end{equation*}
    In our case, we have
    $\dkls{P_1}{P_{-1}} = \delta\log\frac{1+\delta}{1-\delta} \le 3\delta^2$ for
    $\delta\in\brk{0, 1/2}$. We set $\delta = 1 / \sqrt{6n}$, which yields the
    final result in one dimension
    \begin{equation*}
      \mathfrak{M}_n(\brk{-1, 1}, \mc{P}, \mc{F}^\kappa_{d=1})
      = \Omega\prn*{
        \frac{1}{\lambdae}\prn*{\frac{L}{\sqrt{n}}}^{\tfrac{\kappa}{\kappa - 1}}.
      }
    \end{equation*}
\end{proof}

\paragraph{Privatizing the lower bound via a packing argument} We now show how
this construction yields a private lower bound via a packing argument. For
$d\ge 1$, considering the ERM problem, the following private lower bound holds.

\begin{theorem}[Private lower bound for ERM]\label{thm:lb-private-large-kappa}
  Let $d \ge 1, \mc{X} = \ball_2^d(R)$, $\S = \crl{\pm e_j}_{j\le d}$, $\kappa \ge 2$
  and $n\in\N$. Let $\mc{P}$ be the set of distributions on $\S$. Assume that
      \begin{equation*}
      2^{\kappa-1} 
      \le
      \frac{\lip}{\sc} \frac{1}{ R^{\kappa - 1}}
      \le   2^{\kappa-1} 
      \sqrt{96 n}.
    \end{equation*}
    Then any $\diffp$-DP algorithm $\A$ has
  \begin{equation*}
    \sup_{\Ds \in \domain^n} \E \left[ \pf_\Ds(\A(\Ds)) - \min_{x \in \xdomain} \pf_\Ds(x) \right] 
    \gtrsim \frac{1}{\lambda^{\frac{1}{\kappa-1}}} \left(\frac{ \lip d}{n \diffp} \right)^{\frac{\kappa}{\kappa-1}}.
  \end{equation*}
\end{theorem}

\begin{proof}
  First, note that it is enough to prove the following lower bound
	\begin{equation}
	\label{eq:dis-lb}
	\sup_{\Ds \in \domain^n} \E \left[ \ltwo{\A(\Ds) - x\opt} \right]
		\gtrsim \frac{1}{\lambda^{\frac{1}{\kappa-1}}} \left(\frac{ \lip d}{n \diffp} \right)^{\frac{1}{\kappa-1}}.
	\end{equation}
	Indeed, this implies that
	\begin{align*}
	\sup_{\Ds \in \domain^n} \E \left[ f(\A(\Ds)) - \min_{x \in \xdomain} f(x) \right]
	& \ge \frac{\lambda}{\kappa} \sup_{\Ds \in \domain^n} \E \left[ \ltwo{\A(\Ds) - x\opt}^\kappa \right] \\
	& \gtrsim \frac{1}{\kappa \lambda^{\frac{1}{\kappa-1}}} \left(\frac{ \lip d}{n \diffp}  \right)^{\frac{\kappa}{\kappa-1}}.
	\end{align*}
	Let us now prove the lower bound~\eqref{eq:dis-lb}. To this end, we consider the function $\f(x;s) \defeq \frac{\lambda 2^{\kappa - 2}}{\kappa}\norm{x}_2^\kappa + \frac{L}{2}\tri{x, s}$ where $\ltwo{s} \le 1$. We now construct $M$ datasets $\Ds_1,\dots,\Ds_M$ as follows. Let $v_1,\dots,v_M \in \crl*{\pm \frac{1}{\sqrt{d}}}^d$ be the Gilbert-Varshimov packing of the hypercube: that is, $M \ge \exp(d/8)$
	and $\dham(v_i, v_j) \ge d/4$ for all $i \neq j$. We define $\Ds_i = (\underbrace{v_i,\dots,v_i}_{d/20\diffp},0,\dots,0)$. Note that $\dham(S_i,S_j) \le d/20 \diffp$ and that $f(x;\Ds_i) = \frac{\lambda 2^{\kappa - 2}}{\kappa}\norm{x}_2^\kappa + \frac{L}{2} \frac{d}{20 n \diffp} \tri{x, v_i} $, hence
	\begin{equation*}
	x_i^\star = -\prn*{\frac{1}{\lambda 2^{\kappa-2}}}^{\tfrac{1}{\kappa - 1}}
	\prn*{\frac{40 n \diffp}{\lip d}}^{\tfrac{\kappa - 2}{\kappa - 1}} \frac{\lip d}{40 n \diffp} v_i.
	\end{equation*}
	Therefore we have
	\begin{align*}
	\ltwo{x_i\opt - x_j\opt}^2
		& \ge \prn*{\frac{1}{\lambda 2^{\kappa-2}}}^{\tfrac{2}{\kappa - 1}}
		\prn*{\frac{40 n \diffp}{\lip d}}^{\tfrac{2(\kappa - 2)}{\kappa - 1}} \frac{\lip^2 d^2}{1600 n^2 \diffp^2} \\
		& \gtrsim  \prn*{\frac{1}{\lambda 2^{\kappa-2}}}^{\tfrac{2}{\kappa - 1}}
		\prn*{\frac{\lip d}{n \diffp}}^{\tfrac{2}{\kappa - 1}}
		\defeq \rho^2.
	\end{align*}
	We are now ready to finish the proof using packing-based arguments~\cite{HardtTa10}. Assume by contradiction there is an $\diffp$-DP algorithm $\A$ such that 
	\begin{equation*}
	\sup_{1 \le i \le M} \E \left[ \ltwo{\A(\Ds_i) - x_i\opt} \right] \le \rho/20.
	\end{equation*}
	Let $B_i = \{x \in \xdomain : \ltwo{x - x_i\opt} \le \rho/2   \}$. Note that the sets $B_i$ are disjoint and that
	Markov inequality implies 
	\begin{equation*}
	\Pr(\A(\Ds_i) \in B_i)
	 	= \Pr( \ltwo{\A(\Ds_i) - x_i\opt} 
	 	\le \rho/2 ) \ge 9/10.
	\end{equation*}
	Thus, the privacy constraint now gives
	\begin{align*}
	1 
		& \ge \sum_{i=1}^M \Pr(\A(x_1) \in B_i) \\
		& \ge \Pr(\A(x_1) \in B_1) +  e^{-d/20}\sum_{i=2}^M \Pr(\A(x_i) \in B_i) \\
		& \ge \frac{9}{10} (1+ e^{-d/20} (M-1)) ,
	\end{align*}
	where the second inequality follows since $\dham(\Ds_i,\Ds_j) \le
        d/20\diffp$. This gives a contradiction for $d \ge 20$ as
        $M \ge \exp(d/8)$. For $d=1$, we can repeat the same arguments with
        $M=2$ to get the desired lower bound.
      \end{proof}

      \paragraph{Reduction from $\diffp$-DP ERM to $\diffp$-DP SCO} We conclude
      the proof of the theorem by proving that SCO under privacy constraints is
      strictly harder than ERM. This is similar to Appendix~C
      in~\cite{BassilyFeTaTh19} but we require it for pure-DP constraints. We
      make this formal in here.

      %
\begin{comment}
\section{Lower bounds: from ERM to SCO}
\label{sec:lb-erm-sco}
In this section, we provide a reduction which shows that a lower bound for
DP-ERM transforms into a lower bound for DP-SCO. For the case of \ed-DP, this
follows from~\cite[Appendix C]{BassilyFeTaTh19}. We extend their reduction to
pure $\diffp$-DP.

\end{comment}

We have the following lemma.
\begin{proposition}\label{sec:lb-erm-sco}
	Let $0 < \beta \le 1/n$.
	Assume $\A$ is an $\frac{\diffp}{2\log(2/\beta)}$-DP algorithm that for a sample 
	$\Ds = (S_1,\ldots, S_n)$ with $S_1^n \simiid P$ achieves with probability $1-\beta/2$ error
	\begin{equation*}
		\pf(\A(\Ds)) -  \min_{x \in \xdomain} \pf(x) \le \gamma.
	\end{equation*}
	Then there is an $\diffp$-DP algorithm $\A'$ such that for any dataset 
	$\Ds \in \domain^n$ has with probability $1-\beta$,
	\begin{equation*}
	\pf_\Ds(\A'(\Ds)) - \min_{x \in \xdomain} \pf_\Ds(x) \le \gamma.
	\end{equation*}	
\end{proposition}

\begin{proof}
	Given the algorithm $\A$, we define $\A'$ as follows.
	For an input $\Ds \in \domain^n$, let $P_\Ds$ be the empirical distribution of $\Ds$.
	Then, $\A'$ proceeds as follows:
	\begin{enumerate}
		\item Sample a new dataset $\Ds_1 = (S_1',\dots,S'_n)$ where $S'_i \sim P_\Ds$
		\item If there is a sample $S_i$ that was sampled more than $k=2\log(2/\beta)$ times, return $0$
		\item Else, return $\A(\Ds_1)$
	\end{enumerate}
	We need to prove that $\A'$ is $\diffp$-DP and that it has the desired utility.
	For utility, note that $\A'$ returns $0$ at step $2$ with probability at most $\beta/2$, since we have for every $1 \le i \le n$
	\begin{align*}
	\Pr\prn*{\ds_i \mbox{~used more than $k$ times}}
		& = \Pr\prn*{\sum_{j=1}^n Z_i \ge k} \\
		& \le 2^{-k} 
		\le \beta^2/2,
	\end{align*}
	where $Z_j \sim \mathsf{Bernoulli}(p)$ with $p=1/n$, and the second inequality follows
	from Chernoff~\cite[][Thm. 4.4]{MitzenmacherUp05} and $\beta \le 1/10$. Applying a union bound over all samples, we get that step $2$ returns $0$ with probability at most $\beta/2$ as $\beta \le 1/n$. Moreover, Algorithm $\A$ fails with probability at most $\beta/2$. 
	Therefore, as $\pf_\Ds(x) = \E_{S \sim P_\Ds} [F(x;S)] $, we have with probability at least $1-\beta$,
	\begin{equation*}
	\pf_\Ds(\A'(\Ds)) - \min_{x \in \xdomain} \pf_\Ds(x) \le \gamma.
	\end{equation*}
	Let us now prove privacy. Assume we run algorithm $\A'$ on two neighboring datasets $\Ds,\Ds'$, and let $\Ds_1,\Ds'_1$ be the datasets produced at step $1$. Let $B$ denote the event that there was a sample $\ds_i$ that was used more than $k$ times (note that this does not depend on the input). Then for any measurable $\cO$,
	\begin{align*}
	\Pr(\A'(\Ds) \in \cO)
		& = \Pr(\A'(\Ds) \in \cO \mid B) \Pr(B) + \Pr(\A'(\Ds) \in \cO \mid B^c) \Pr(B^c) \\
		& \le e^\diffp \Pr(\A'(\Ds') \in \cO \mid B) \Pr(B) + \Pr(\A'(\Ds') \in \cO \mid B^c) \Pr(B^c) \\
		& \le e^{\diffp} \Pr(\A'(\Ds') \in \cO),
	\end{align*}
	where the first inequality follows from group privacy since $\dham(\Ds_1,\Ds'_1) \le k$ and $\A$ is $\diffp/k$-DP. This completes the proof.
	
\end{proof}

    \subsubsection{Proof of~\Cref{thm:lb-small-kappa}}\label{app:prf-thm-lb-small-kappa}

    \restateLbSmallKappa*

  \begin{proof} We follow the same reduction that we used in the proof of
    Theorem~\ref{thm:lb-non-private-large-kappa}. For $\delta \in \brk{0, 1/2}$, we again
    consider $P_v = 1$ with probability $\tfrac{1+\delta v}{2}$ and $-1$
    otherwise. For $a \in [0, 1]$ to be defined later, we construct the
    following function
    \begin{equation*}
      F(x;+1) = \begin{cases}
        \abs{x-a} & \mbox{~~if~~} x \le a \\
        \abs{x-a}^\kappa & \mbox{~~if~~} x \ge a
      \end{cases}
      \mbox{~~and~~}
            F(x;-1) = \begin{cases}
        \abs{x+a}^\kappa & \mbox{~~if~~} x \le -a \\
        \abs{x+a} & \mbox{~~if~~} x \ge -a
      \end{cases}
    \end{equation*}

    \underline{Computing the separation.} First, let us compute the separation
    $d_{\mathsf{opt}}(v, v', \mc{X})$. We will then choose $a$ to ensure $f_v$
    has $\kappa$-growth. By symmetry, assume $v=1$. $f_v$ is increasing on
    $[a, 1]$ and decreasing on $[-1, -a]$, thus the minimum belongs to
    $\brk{-a, a}$ and by inspection, is attained at $x = a$ with value
    $a(1-\delta)$. Similarly, the minimum of $f_{+1}(x) + f_{-1}(x)$ is attained
    on $\brk{-a, a}$ with value $2a$. This yields
  \begin{equation*}
    d_{\mathsf{opt}}(v, v', \mc{X}) = 2a - 2a(1-\delta) = 2a\delta.
  \end{equation*}

  Let us now pick $a$ such that $f_v$ has $\kappa$-growth. Again, by symmetry we
  only treat the $v=1$ case. We have
  \begin{equation*}
    \mbox{for~} x\ge a, f_v(x) - f_v^\star = \frac{1+\delta}{2}(x-a)^\kappa
    + \frac{1-\delta}{2}(x+a) - a(1-\delta) =
    \frac{1+\delta}{2}(x-a)^{\kappa} + \frac{1-\delta}{2}(x-a) \ge \abs{x-a}^\kappa,
  \end{equation*}
  where the last inequality is because $(x-a) \le 1$ and so
  $(x-a) \ge (x-a)^\kappa$ for $\kappa > 1$. In the second case, we have
  \begin{equation*}
    \mbox{for~} x\in\brk{-a, a}, f_v(x) - f_v^\star = \delta(a-x).
  \end{equation*}
  It holds that $\delta(a-x)\ge (a-x)^\kappa$ for all $x\in\brk{-a, a}$ iff
  $a \le \tfrac{1}{2}\delta^{\tfrac{1}{\kappa -1}}$. As a result, we set
  $a = \tfrac{1}{2}\delta^{\tfrac{1}{\kappa - 1}}$. Finally, for $x \in \brk{-1, -a}$, we define
  \begin{equation*}
    h(x) \defeq \frac{1+\delta}{2}\abs{x-a} + \frac{1-\delta}{2}\abs{x+a}^\kappa - a(1-\delta)
    - \frac{1}{\kappa}\abs{x-a}^\kappa \mbox{~~for~~} x\in\brk{-1, -a}.
  \end{equation*}
  We wish to prove that $h(x) \ge 0$. First of, note that
  $h(-a) = \delta^{\tfrac{\kappa}{\kappa - 1}}(\tfrac{1}{2} + \tfrac{1}{2} -
  \tfrac{1}{\kappa}) > 0$, whenever $\kappa > 1$. Let us show that $h(x)$ is
  decreasing on $[-1, -a]$ which suffices to conclude the proof. We have
  \begin{equation*}
    h'(x) = -\frac{1+\delta}{2} - \frac{\kappa(1-\delta)}{2}\abs{x+a}^{\kappa-1}
    + \abs{x-a}^{\kappa-1}.
  \end{equation*}
  First of, note that $h'(-a) = -\tfrac{1+\delta}{2} + \delta \le 0$ and
  $h'(-1) < 0$, thus it suffices to show that if $h'$ has an extremum then is it
  negative. An extremum of this function is a point $x^\star$ such that
  \begin{equation*}
    \abs{a - x^\star} = \prn*{\frac{\kappa(1-\delta)}{2}}^{\tfrac{1}{\kappa - 2}}
    \abs{a + x^\star},
  \end{equation*}
  which yields that
  \begin{equation*}
    h'(x^\star) = \abs{a+x^\star}^{\kappa - 1}\prn*{\frac{\kappa(1-\delta)}{2}}\brk*{
      \prn*{\frac{\kappa(1-\delta)}{2}}^{\tfrac{1}{\kappa - 2}} - 1
    } - \frac{1+\delta}{2} \le 0,
  \end{equation*}
  as $\kappa \le 2$. This calculation shows that $f_v$ has $(1,
  \kappa)$-growth. Finally note that the function is $\kappa \le 2$-Lipschitz as
  desired.

  \underline{Lower bounding the testing error.} It remains to choose the value
  of $\delta$. Since we require a lower bound under privacy constraints, in
  contrast to the one-dimensional section of the proof
  of~\Cref{thm:lb-non-private-large-kappa}, we require the following privatized
  version of Le Cam's lemma from~\cite{BarberDu14a}
    \begin{proposition}{\cite[Thm. 2]{BarberDu14a}} Let $\msf{A}\in\mc{A}^\diffp$ be an $\epsilon$-DP
      mechanism from $\S^n\to\mc{X}$. It holds that
      \begin{equation*}
        \inf_{\psi:\mc{X}\to\crl{-1, 1}} \inf_{\msf{A}\in\mc{A}^\eps} \P(\psi(\msf{A}(S_1^n)) \neq V) \ge
      \frac{1}{2}\prn*{1 - \min\crl*{2n\diffp\tvnorm{P_1 - P_{-1}}, \tvnorm{P_{-1}^n - P_1^n}}}.
    \end{equation*}
    \end{proposition}
    With this result, we set
    $\delta = \max\crl{1/\sqrt{6n}, 1/(2\sqrt{3}n\diffp)}$ and lower bound
    $\max\crl{a, b}$ by $a+b$ for readability, which concludes the proof of the
    theorem.
\end{proof}

\subsection{Proof for Section~\ref{sec:lb-appr}}
\label{app:prf-lb-approx}

\restatePropReduction*

\begin{proof}[Proof of Proposition~\ref{prop:reduction-growth-to-erm}]
  Let us first show how to construct the mechanism $\msf{A}'$. Let $k \in \N$ be
  such that
  $k \ge \log_2\brk*{\frac{\kappa^{\tfrac{1}{\kappa -
          1}}L^{\tfrac{\kappa}{\kappa-1}}}{2^{2\kappa - 3}\Delta(n, L, \diffp/k,
      \delta)}}$ and let $\crl{\lambda_i}_{i\in\brk{k}}$ be a collection of
  positive scalars. Set $x_0 \in \mc{X}$, for $i \in \crl{1, \ldots, k}$
  \begin{align*}
    & \mbox{define~~} G_i(x;s) = \f(x;s) + \frac{\lambda_i\cdot 2^{\kappa - 2}}{\kappa}\norm{x-x_{i-1}}_2^\kappa,
      \mc{Y}_i \defeq \crl*{x\in\mc{X}: \norm{x-x_{i-1}}_2 \le \prn*{\frac{L\kappa}{\lambda_i2^{\kappa - 2}}}^{\tfrac{1}{\kappa - 1}}} \\
    & \mbox{~~and set~~} x_i = \msf{A}(\mc{S}, G_i, \mc{Y}_i), \mbox{~with privacy~~} (\diffp/k, \delta/k).
  \end{align*}
  Finally, define $\msf{A}'(\mc{S}) = x_k$. Standard composition
  theorems~\cite{DworkRo14} guarantee that $\msf{A}'$ is $(\diffp, \delta)$-DP. Let
  us analyze its utility; we drop the dependence of $\Delta$ on other variables
  when clear from context. First of, since $\kappa$ is a constant, note that
  $G_i$ is $c_0\lip$-Lipschitz with $c_0 < \infty$ a numerical constant.  For
  simplicity, we define $g_i(x) \defeq \frac{1}{n}\sum_{s\in\mc{S}}G_i(x;s)$ and
  $x_i^\star = \argmin_{x\in\mc{Y}_i} g_i(x)$. It holds that $g_i$ is
  $(\lambda_i2^{\kappa - 2}, \kappa)$-uniformly-convex and thus the following
  growth condition holds
  \begin{equation*}
    \frac{\lambda_i}{\kappa}\E\norm{x_i - x_i^\star}_2^\kappa
    \le \E\brk{g_i(x_i)} - g_i(x_i^{\star}) \le \frac{1}{\lambda_i^{\tfrac{1}{\kappa - 1}}}\Delta.
  \end{equation*}

  Also note that for any point $y \in \mc{Y}_i$, it holds that
  \begin{equation*}
    f_\mc{S}(x_i^\star) - f(y) \le \frac{\lambda_i2^{\kappa - 2}}{\kappa}\norm{x_{i-1} - y}_2^\kappa.
  \end{equation*}

  Finally, let us bound the distance to the optimum of $\femp$ at the final
  iterate. We have
  \begin{equation*}
    \frac{\lambda_k}{\kappa}\norm{x_k-x_k^\star}_2^\kappa \le g_k(x_k) - g_k(x_k^\star)
    \le c_0L\norm{x_k - x_k^\star}_2 \mbox{~~which yields~~}
    \norm{x_k - x_k^\star}_2 \le \prn*{\frac{c_0L\kappa}{\lambda_k}}^{\tfrac{1}{\kappa - 1}}.
  \end{equation*}
  
  Let us put the pieces together: for $\lambda > 0$ to be determined later and
  $\nu = \kappa - 1$, set $\lambda_i = 2^{-\nu i} \lambda$. After $k$ rounds and
  denoting $x_0^\star = \inf_{x\in\mc{X}}f_{\mc{S}}(x)$, we have
  \begin{align*}
    \E\brk{\ff_{\mc{S}}(x_k)} - \ff_{\mc{S}}(x^\star)
    & = \sum_{i=1}^k \E\brk*{\femp(x_i^\star) - \femp(x_{i-1}^\star)}
      + \E\brk*{\femp(x_k) - \femp(x_k^\star)} \\
    & \le \sum_{i=1}^k \frac{\lambda_i2^{\kappa - 2}}{\kappa}\E\norm{x_{i-1} - x_{i-1}^\star}_2^\kappa
      + \lip\prn*{\frac{c_0\lip \kappa}{\lambda_k}}^{\tfrac{1}{\kappa -1}} \\
    & \le \frac{\lambda D^\kappa}{\kappa} +
      \sum_{i=2}^k \frac{\lambda_i2^{\kappa - 2}}{\lambda_{i-1}^{\tfrac{\kappa}{\kappa - 1}}}\Delta
      + \lip\prn*{\frac{c_0\lip \kappa}{\lambda_k}}^{\tfrac{1}{\kappa -1}} \\
    & = \frac{\lambda D^\kappa}{\kappa} +
      \frac{\Delta 2^{\kappa - 2}}{\lambdae} \sum_{i=2}^k 2^{-\tfrac{\nu}{\kappa - 1}(i - \kappa)}
      + \lip\prn*{\frac{c_0\lip \kappa}{\lambda}}^{\tfrac{1}{\kappa -1}}2^{-\tfrac{\nu}{\kappa - 1}k} \\
    & \le \frac{\lambda D^\kappa}{\kappa} + 2^{2\kappa - 3}\frac{\Delta}{\lambdae}
      + \frac{\kappa^{\tfrac{1}{\kappa - 1}}(c_0\lip)^{\tfrac{\kappa}{\kappa - 1}}2^{-k}}{\lambdae}.
  \end{align*}
  Finally, note that
  \begin{equation*}
    k \ge \ceil*{\log_2\brk*{\frac{\kappa^{\tfrac{1}{\kappa - 1}}(c_0L)^{\tfrac{\kappa}{\kappa-1}}}{2^{2\kappa - 3}\Delta}}}
    \mbox{~~so that~~}
    \frac{\kappa^{\tfrac{1}{\kappa - 1}}(c_0\lip)^{\tfrac{\kappa}{\kappa - 1}}2^{-k}}{\lambdae} \le
    2^{2\kappa - 3}\frac{\Delta}{\lambdae}.
  \end{equation*}
  It then holds that
  \begin{equation*}
    \E\brk{\femp(x_k)} - \femp(x^\star) \le \lambda\frac{D^\kappa}{\kappa} + 4^{\kappa - 1}\Delta\frac{1}{\lambdae}.
  \end{equation*}

  It remains to pick $\lambda$ to minimize the upper bound above. A calculation yields that for $a, b \ge 0$
  \begin{equation*}
    \inf_{\nu \ge 0} a\nu + \frac{b}{\nu^{\tfrac{1}{\kappa - 1}}} = (\kappa - 1)^{1/\kappa}a^{1/\kappa} b^{(\kappa - 1)/\kappa}
    \brk*{\kappa - 1 + \frac{1}{\kappa - 1}}
    \mbox{~~at~~} \nu^\star = \prn*{\frac{b}{a(\kappa - 1)}}^{\tfrac{\kappa - 1}{\kappa}}.
  \end{equation*}
  Setting
  $\lambda = 4^{\tfrac{(\kappa - 1)^2}{\kappa}}\prn{\tfrac{\Delta
      \kappa}{D^\kappa(\kappa - 1)}}^{(\kappa - 1)/\kappa}$ yields the regret
  bound
  \begin{equation*}
    \E\brk{\femp(x_k)} - \femp(x^\star) \le O(1)D\Delta^{\tfrac{\kappa - 1}{\kappa}}.
  \end{equation*}
\end{proof}

\begin{proof} Consider the reduction of
  Proposition~\ref{prop:reduction-growth-to-erm}. For $c_1 < \infty$ to be
  determined later, assume by contradiction that there exists an
  $(\diffp, \delta)$ mechanism such that
  \begin{equation*}
    \Delta(n, L, \diffp, \delta) \le c_1\prn*{\frac{L\sqrt{d}}{n\diffp}}^{\tfrac{\kappa}{\kappa - 1}}.
  \end{equation*}
  Setting
  $k = \ceil{4\log_2(n\diffp /
    \sqrt{d})\log\log_2((n\diffp/\sqrt{d})^{\kappa/(\kappa - 1)})}$, the
  condition holds and the result of
  Proposition~\ref{prop:reduction-growth-to-erm} guarantees that there exists a
  numerical constant $c_2 < \infty$ and a mechanism $\msf{A}'$ such that
  \begin{equation*}
    \E\brk{\femp(\msf{A}'(\mc{S}))} - \inf_{x'\in\mc{X}}\femp(x')
    \le c_2c_1^{\tfrac{\kappa - 1}{\kappa}}k\rad\frac{L\sqrt{d}}{n\diffp}.
  \end{equation*}
  
  However, Theorem 5.3 in~\cite{BassilySmTh14} guarantees that there exists
  $c_3 > 0$ such that for any $(\diffp, \delta)$-DP mechanism $\msf{A}''$, it
  must hold
  \begin{equation*}
    c_3 \lip\rad \frac{\sqrt{d}}{n\diffp} \le \E\brk{\femp(\msf{A}''(\mc{S}))} - \femp(x^\star).
  \end{equation*}
  Setting $c_1 = \tfrac{1}{2}\prn*{\tfrac{c_3}{kc_2}}^{\tfrac{\kappa}{\kappa-1}}$ yields a contradiction
  and the desired lower bound by noting that $k$ consists only of log factors.
\end{proof}

\end{document}